%% file: main.tex
\documentclass{article}

    \usepackage[final]{neurips_2022}

\usepackage[utf8]{inputenc} 
\usepackage[T1]{fontenc}    
\usepackage{hyperref}       
\usepackage{url}            
\usepackage{booktabs}       
\usepackage{amsfonts}       
\usepackage{nicefrac}       
\usepackage{microtype}      
\usepackage{xcolor}         

\include{header}

\newcommand{\todo}[1]{\textcolor{red}{}}

\title{Off-Policy Evaluation for Action-Dependent Non-Stationary Environments}

\author{%
  Yash Chandak
  \\
  University of Massachusetts
  \And
  Shiv Shankar
  \\
  University of Massachusetts
  \And
  Nathaniel D. Bastian
  \\
  United States Military Academy
  \And 
  Bruno Castro da Silva
  \\
  University of Massachusetts
  \And
  Emma Brunskill
  \\
   Stanford University
  \And
  Philip S. Thomas
  \\
  University of Massachusetts
}

\begin{document}

\etocdepthtag.toc{mtchapter}
\etocsettagdepth{mtchapter}{subsection}
\etocsettagdepth{mtappendix}{none}

\maketitle

\begin{abstract}

Methods for sequential decision-making are often built upon a foundational assumption that the underlying decision process is stationary. This limits the application of such methods because real-world problems are often subject to changes due to external factors (\textit{passive} non-stationarity), changes induced by interactions with the system itself (\textit{active} non-stationarity), or both (\textit{hybrid} non-stationarity). In this work, we take the first steps towards the fundamental challenge of on-policy and off-policy evaluation amidst structured changes due to active, passive, or hybrid non-stationarity. Towards this goal, we make a \textit{higher-order stationarity} assumption such that non-stationarity results in changes over time, but the way changes happen is fixed. We propose, OPEN, an algorithm that uses a double application of counterfactual reasoning and a novel importance-weighted instrument-variable regression to obtain both a lower bias and a lower variance estimate of the structure in the changes of a policy's past performances. Finally, we show promising results on how OPEN can be used to predict future performances for several domains inspired by real-world applications that exhibit non-stationarity. 
%
%
%
\end{abstract}
 \section{Introduction}
Methods for sequential decision making are often built upon a foundational assumption that the underlying decision process is stationary \citep{SuttonBarto2}. 
%
%
%
While this assumption was a cornerstone when laying the theoretical foundations of the field, and while is often reasonable, it is seldom true in practice and can be unreasonable \citep{dulac2019challenges}.
%
%
%
%
Instead, real-world problems are subject to non-stationarity that can be broadly classified as (a) \textit{Passive:} where the changes to the system are induced only by external (exogenous) factors, (b) \textit{Active:} where the changes result due to the agent's past interactions with the system, and (c) \textit{Hybrid:} where both passive and active changes can occur together \citep{khetarpal2020towards}. 
%
%


There are many applications that are subject to active, passive, or hybrid non-stationarity, and where the stationarity assumption may be unreasonable. 
Consider methods for automated healthcare where we would like to use the data collected over past decades to find better treatment policies.
%
%
In such cases, not only might there have been passive changes due to healthcare infrastructure changing over time, but active changes might also occur because of public health continuously evolving based on the treatments made available in the past, thereby resulting in hybrid non-stationarity.
Similar to automated healthcare, other applications like online education, product recommendations, and in fact almost all human-computer interaction systems need to not only account for the continually drifting behavior of the user demographic, but also how the preferences of users 
may change due to interactions with the system \citep{theocharous2020reinforcement}.
%
%
%
%
Even social media platforms need to account for the partisan bias of their users that change due to both external political developments and increased self-validation resulting from previous posts/ads suggested by the recommender system itself \citep{cinelli2021echo,gillani2018me}.
Similarly, motors in a robot suffer wear and tear over time not only based on natural corrosion but also on how vigorous past actions were. 
%
%

However, conventional off-policy evaluation methods \citep{precup2000eligibility,jiang2015doubly,xie2019towards} predominantly focus on the stationary setting.  
These methods assume availability of either (a) \textit{resetting assumption} to sample multiple sequences of interactions from a stationary environment with a fixed starting state distribution (i.e., episodic setting), or (b) \textit{ergodicity assumption} such that interactions can be sampled from a steady-state/stationary distribution (i.e., continuing setting). For the problems of our interest, methods based on these assumptions may not be viable. 
For e.g., in automated healthcare, we have a single long history for the evolution of public health, which is neither in a steady state distribution nor can we reset and go back in time to sample another history of interactions.

As discussed earlier, because of non-stationarity the transition dynamics and reward function in the future can be different from the ones in the past, and these changes might also be dependent on past interactions.
In such cases, how do we even address the fundamental challenge of \textit{off-policy evaluation}, i.e., using data from past interactions to estimate the performance of a new policy in the future?
Unfortunately, if the underlying changes are arbitrary, even amidst only passive non-stationarity it may not be possible to provide non-vacuous predictions of a policy's future performance \citep{chandak2020towards}.

%
%

Thankfully, for many real-world applications there might be (unknown) structure in the underlying changes.
In such cases, can the \textit{effect} of the underlying changes on a policy's performance be inferred, \textit{without} requiring estimation of the underlying model/process?
Prior work has only shown that this is possible in the passive setting. 
This raises the question that we aim to answer:
%
\begin{center}
\textit{How can one provide a unified procedure for (off) policy evaluation amidst active,\\ passive, or hybrid non-stationarity, when  the underlying changes are structured?}
\end{center}

\textbf{Contributions:} 
%
%
%
To the best of our knowledge, our work presents the first steps towards addressing the fundamental challenge of off-policy evaluation amidst structured changes due
to active or hybrid non-stationarity.
%
Towards this goal, we make a \textit{higher-order stationarity} assumption, under which the non-stationarity can result in changes over time, but the way changes happen is fixed.
Under this assumption, we propose  a model-free method that can infer the \textit{effect} of the underlying non-stationarity on the past performances and use that to predict the future performances for a given policy.
We call the proposed method OPEN: \underline{o}ff-\underline{p}olicy \underline{e}valuation for \underline{n}on-stationary domains.
On domains inspired by real-world applications, we show that OPEN often provides significantly better results not only in the presence of active and hybrid non-stationarity, but also for the passive setting where it even outperforms previous methods designed to handle only passive non-stationarity.

OPEN primarily relies upon two key insights: \textbf{(a)} For active/hybrid non-stationarity, as the underlying changes may dependend on past interactions, the structure in the changes observed when executing the data collection policy can be different than if one were to execute the evaluation policy.
To address this challenge, OPEN makes uses counterfactual reasoning twice and permits reduction of this off-policy evaluation problem to an auto-regression based forecasting problem.
\textbf{(b)} Despite reduction to a more familiar auto-regression problem, in this setting naive least-squares based estimates of parameters for auto-regression suffers from high variance and can even be asymptotically biased.
Finally, to address this challenge, OPEN uses a novel importance-weighted instrument-variable (auto-)regression technique to obtain asymptotically consistent and lower variance parameter estimates.

\section{Related Work}
Off-policy evaluation (OPE) is an important aspect of reinforcement learning \citep{precup2000eligibility,thomas2015higha,SuttonBarto2} and various techniques have been developed to construct efficient estimators for OPE \citep{jiang2015doubly,thomas2016data,munos2016safe,harutyunyan2016q,espeholt2018impala, xie2019towards}. However, these work focus on the stationary setting.
Similarly, there are various methods for tackling non-stationarity in the bandit setting \citep{moulines2008,besbes2014stochastic,seznec2018rotting,wang2019aware}. 
In contrast, the proposed work focuses on methods for sequential decision making.

Literature on off-policy evaluation amidst non-stationarity for sequential decision making is sparse.
Perhaps the most closely related works are by \citet{thomas2017predictive,chandak2020optimizing,xie2020deep,poiani2021meta,liotet2021lifelong}.
While these methods present an important stepping stone, such methods are for passive non-stationarity and, as we discuss using the toy example in Figure \ref{fig:badpassive}, may result in undesired outcomes if used as-is in real-world settings that are subject to active or hybrid non-stationarity.
%
%
%

\begin{wrapfigure}[7]{l}{0.4\textwidth}
    \centering    
    \vspace{-10pt}
    \includegraphics[width=0.38\textwidth]{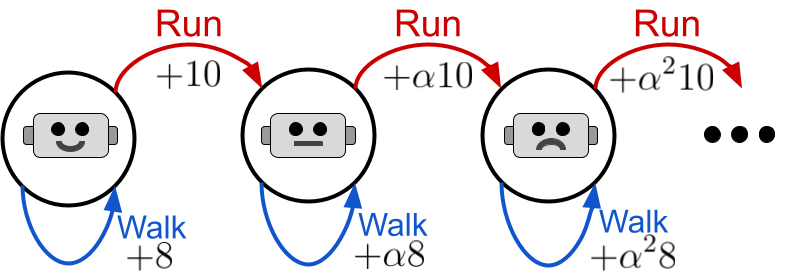}
    \caption{RoboToy domain.}
    \label{fig:badpassive}
\end{wrapfigure}
%

Consider a robot that can perform a task each day either by `walking' or `running'.
    A reward of $8$ is obtained upon completion using `walking', but `running' finishes the task quickly and results in a reward of $10$. However, `running' wears out the motors, thereby increasing the time to finish the task the next day and reduces the returns for \textit{both} `walking' and `running' by a small factor, $\alpha \in (0,1)$.
    %
    %
    
    Here, methods for tackling \textit{passive} non-stationarity will track the best policy under the assumption that the changes due to damages are because of external factors and would fail to attribute the cause of damage to the agent's decisions.
    Therefore, as on any given day `running' will always be better, every day these methods will prefer `running' over `walking' and thus \textit{aggravate} the damage.
    Since the outcome on each day is dependent on decisions made during  previous days this leads to active non-stationarity,  where `walking' is better in the long run.
    Finding a better policy first requires a method to evaluate a policy's (future) performance, which is the focus of this work.

Notice that the above problem can also be viewed as a task with effectively a \textit{single} lifelong episode. However, as we discuss later in Section \ref{sec:ass}, approaches such as modeling the problem as a large stationary POMDP or as a continuing average-reward MDP with a single episode may not be viable.
Further, non-stationarity can also be observed in multi-agent systems and games due to different agents/players interacting with the system.
However, often the goal in these other areas is to search for (Nash) equilibria, which may not even exist under hybrid non-stationarity. 
Non-stationarity may also result due to artifacts of the learning algorithm even when the problem is stationary.
While relevant, these other research areas are distinct from our setting of interest and we discuss them and others in more detail in Appendix \ref{sec:related}.

\section{Non-Stationary Decision Processes}

We build upon the formulation used by past work \citep{xie2020deep, chandak2020optimizing} and consider that the agent interacts with a lifelong sequence of partially observable Markov decision processes (POMDPs), $(M_i)_{i=1}^\infty$.
%
However, unlike prior problem formulations, we account for active and hybrid non-stationarity by considering POMDP $M_{i+1}$ to be dependent on \textit{both} on the POMDP $M_i$ and the decisions made by the agent while interacting with $M_i$.
We provide a control graph for this setup in Figure \ref{fig:controlgraph}.
%
%
%
For simplicity of presentation, we will often ignore the dependency of $M_{i+1}$ on $M_{i-k}$ for $k>0$, although our results can be extended for settings with $k>0$.

\textbf{Notation:} Let $\mathcal M$ be a finite set of POMDPs.
Each POMDP $M_i \in \mathcal M$ is a tuple $(\mathcal O, \mathcal S, \mathcal A, \Omega_i, P_i, R_i, \mu_i)$, where $\mathcal O$ is the set of observations, $\mathcal S$ is the set of states, and $\mathcal A$ is the set of actions, which are the same for all the POMDPs in $\mathcal M$.
For simplicity of notation, we assume $\mathcal M, \mathcal S, \mathcal O, \mathcal A$ are finite sets, although our results can be extended to settings where these sets are infinite or continuous.
Let $\Omega_i: \mathcal S \times \mathcal O \rightarrow [0,1]$ be the observation function, $P_i: \mathcal S \times \mathcal A \times \mathcal S \rightarrow [0,1]$ be the transition function, $\mu_i: \mathcal S \rightarrow [0,1]$ be the starting state distribution, and $R_i: \mathcal S \times \mathcal A \rightarrow [-R_{\max}, R_{\max}]$ be the reward function with $0 \leq R_{\max} < \infty$. 

Let $\pi: \mathcal O \times \mathcal A \rightarrow [0,1]$ be any policy and $\Pi$ be the set of all policies.
Let $H_i \coloneqq (O_i^t, A_i^t, R_i^t)_{t=1}^T$ be a sequence of at most $T$ interactions in $M_i$, where $O_i^t, A_i^t, R_i^t$ are the random variables corresponding to the observation, action, and reward at the step $t$.
Let $G_i\coloneqq\sum_{t=1}^T R_i^{t}$ be an observed return and $J_i(\pi) \coloneqq \mathbb{E}_{\pi}[G_i|M_i]$ be the performance of $\pi$ on $M_i$.
Let $\mathcal H$ be the set of possible interaction sequences, and finally let $\mathcal T: \mathcal M \times \mathcal H \times \mathcal M \rightarrow [0,1]$ be the transition function that governs the non-stationarity in the POMDPs.
That is, $\mathcal T(m,h,m')=\Pr(M_{i+1}{=}m'|M_i{=}m,H_i{=}h)$. 
%
%
\newpage
\begin{wrapfigure}[12]{r}{0.52\textwidth}
    \centering
    \includegraphics[width=0.52\textwidth]{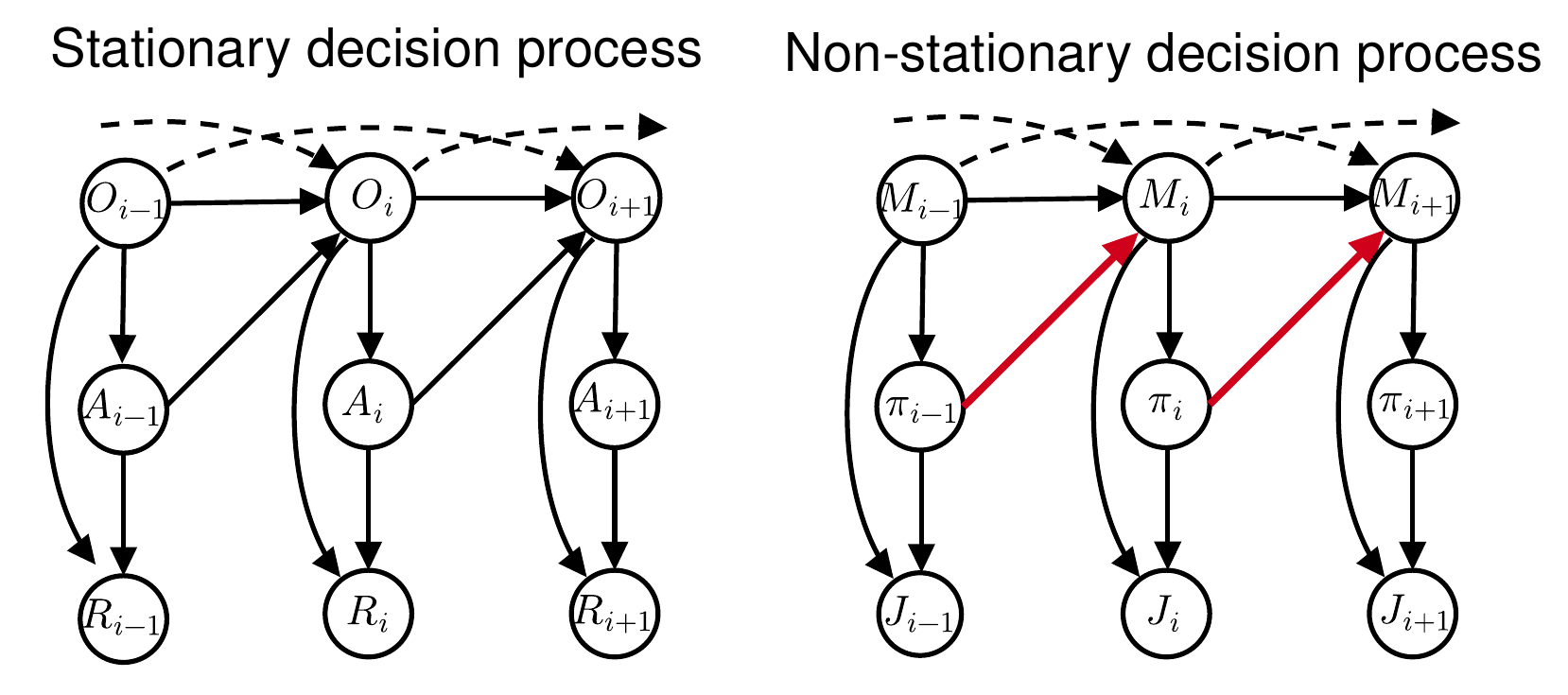}
    \vspace{-15pt}
    \caption{Control Graph for the non-stationary process. See text for symbol definitions. }
    \label{fig:controlgraph}
\end{wrapfigure}

    Figure \ref{fig:controlgraph} \textbf{(Left)} depicts the control graph for a stationary POMDP, where each column corresponds to one time step. 
    Here, \textit{multiple, independent} episodes from the \textit{same} POMDP can be resampled.
    \textbf{(Right)} Control graph that we consider for a non-stationary decision process, where each column corresponds to one episode.
    Here, the agent interacts with a \textit{single} sequence of related POMDPs $(M_i)_{i=1}^n$.
    Absence or presence of the red arrows indicates whether the change from $M_{i}$ to $M_{i+1}$ is independent of the decisions in $M_i$ (passive non-stationarity) or not (active non-stationarity).
    %
    %

\textbf{Problem Statement: } 
We look at the fundamental problem of evaluating the performance of a policy $\pi$ in the presence of non-stationarity.
Let $(H_i)_{i=1}^n$ be the data collected in the past by interacting using policies $(\beta_i)_{i=1}^n$.
%
Let $D_n$ be the dataset consisting of $(H_i)_{i=1}^n$ and the probabilities of the actions taken by $(\beta_i)_{i=1}^n$.
Given $D_n$, we aim to evaluate the expected future performance of $\pi$ if it is deployed for the \textit{next} $L$ episodes (each a different POMDP), that is
%
$  
   \mathscr{J}(\pi) \coloneqq \mathbb{E}_\pi\left[\sum_{k=n+1}^{n+L}  J_k(\pi) \middle| (H_i)_{i=1}^n \right]. \label{eqn:obj}
  $
%
We call it the \textit{on-policy} setting if $\forall i, \beta_i = \pi$, and the \textit{off-policy} setting otherwise.
Notice that even in the on-policy setting, naively aggregating observed performances from $(H_i)_{i=1}^n$ may not be indicative of $\mathscr J(\pi)$ as $M_{k}$ for $k>n$ may be different than $M \in (M_i)_{i=1}^n$ due to non-stationarity.

\section{Understanding Structural Assumptions} \label{sec:ass}
A careful reader would have observed that 
instead of considering interactions with a sequence of POMDPs $(M_i)_{i=1}^n$ that are each dependent on the past POMDPs and decisions, an equivalent setup might have been to consider a `chained' sequence of interactions $(H_1,H_2,...,H_n)$ as a \textit{single} episode in a `mega' POMDP comprised of all $M \in \mathcal M$.
Consequently, $\mathscr J(\pi)$ would correspond to the expected future return given $(H_i)_{i=1}^n$.
%
%
%
Tackling this single long sequence of interactions using the continuing/average-reward setting is not generally viable because methods for these settings rely on an ergodicity assumption (which implies that all states can always be revisited) that may not hold in the presence of non-stationarity. 
%
%
%
For instance, in the earlier example of automated healthcare, it is not possible to revisit past years.  
%


%
%
%

To address the above challenge, we propose introducing a different structural assumption. 
%
%
%
%
Particularly, framing the problem as a sequence of POMDPs allows us to split the single sequence of interactions into multiple (dependent) fragments, with additional structure linking together the fragments. 
%
%
Specifically, we make the following intuitive assumption.


\begin{ass}
    \thlabel{ass:fixedf}
    $\forall m \in \mathcal M$ such that the performance $J(\pi)$ associated with $m$ is $j$, 
    \begin{align}
        \forall \pi, \pi' \in \Pi^2, \forall i,\, \Pr(J_{i+1}(\pi)=j_{i+1}| M_i=m; \pi') = \Pr(J_{i+1}(\pi)=j_{i+1}|J_i(\pi)=j; \pi'). \label{eqn:nstransitiona}
    \end{align}
\end{ass}
%
\thref{ass:fixedf} characterizes  the probability that $\pi$'s performance will be $j_{i+1}$ in the $i+1^\text{th}$ episode when the policy $\pi'$ is executed in the $i^\text{th}$ episode.
To understand \thref{ass:fixedf} intuitively, consider a `meta-transition' function that characterizes $\Pr(J_{i+1}(\pi)|J_i(\pi),\pi')$ similar to how the standard transition function in an MDP characterizes $\Pr(S_{t+1}|S_t,A_t)$.
While the underlying changes actually happen via $\mathcal T$, \thref{ass:fixedf} imposes the following two conditions: \textbf{(a)} A \textit{higher-order stationarity} condition on the meta-transitions under which non-stationarity can result in changes over time, but \textit{the way the changes happen is fixed}, 
and \textbf{(b)} Knowing the past performance(s) of a policy $\pi$ provides \textit{sufficient} information for the meta-transition function to model how the performance will change upon executing any (possibly different) policy $\pi'$. 
For example, in the earlier toy robot domain, given the current performance there exists an (unknown) oracle that can predict the performance for the next day if the robot decides to `run'/`walk'.
%
%
%
%
%

\thref{ass:fixedf} is beneficial as it implicitly captures the effect of both the underlying passive and active non-stationarity by modeling the conditional distribution of the performance $J_{i+1}(\pi)$ given $J_i(\pi)$, when executing any (different) policy $\pi'$. 
At the same time, notice that it generalizes \textbf{(a)} 
the  stationary setting, where
$ \forall \pi \in \Pi, \,   \forall i>0, J_{i+1}(\pi) = J_i(\pi)$, and \textbf{(b)} only passive non-stationarity, which is a special case of \eqref{eqn:nstransitiona} wherein $\pi'$ does not influence the outcome, i.e.,
\begin{align}
    \forall \pi_a, \pi_b \in \Pi^2,\,\forall i>0,  \, \Pr(J_{i+1}(\pi)=J_{i+1}|J_i(\pi)=j; {\color{red} \pi_a}) = \Pr(J_{i+1}(\pi)=J_{i+1}|J_i(\pi)=j; {\color{red} \pi_b}).
\end{align}

\begin{rem}
In some cases, it may be beneficial to relax \thref{ass:fixedf} such that instead of using $\Pr(J_{i+1}(\pi)|J_i(\pi); \pi')$ in \eqref{eqn:nstransitiona}, one considers  $\Pr(J_{i+1}(\pi)|(J_{i-k}(\pi))_{k=0}^p; \pi')$. 
This can be considered similar to the p-Markov MDP where the transitions are characterized using $\Pr(S_{t+1}|(S_{t-i})_{i=0}^p, A_t)$. 
%
%
%
%
While we consider this general setting for our empirical results, for simplicity, to present the key ideas we will consider \eqref{eqn:nstransitiona}.
%
    %
We provide a detailed discussion on cases where we expect such an assumption to be \emph{(in)valid}, and also other potential assumptions in Appendix \ref{sec:assumption}.
\thlabel{rem:p}
\end{rem}

\section{Model-Free Policy Evaluation}
\label{sec:MFPE}
\begin{wrapfigure}[8]{r}{0.4\textwidth}
\vspace{-15pt}
    \centering
    \includegraphics[width=0.4\textwidth]{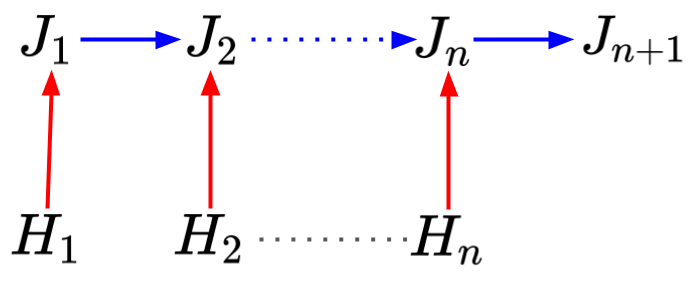}
    \vspace{-10pt}
    \caption{High-level idea.}
    \label{fig:idea}
\end{wrapfigure}

In this section we discuss how under \thref{ass:fixedf}, we can perform model-free off-policy evaluation amidst passive, active, or hybrid non-stationarity.
The high level idea can be decomposed into the following: \textbf{(a)} Obtain estimates of $(J_i(\pi))_{i=1}^n$ using  $(H_i)_{i=1}^n$ ({\color{red}red} arrows in Figure \ref{fig:idea}),  and \textbf{(b)} Use the estimates of $(J_i(\pi))_{i=1}^n$ to infer the \textit{effect} of the underlying non-stationarity on the performance, and use that to predict $(J_i(\pi))_{i=n+1}^{n+L}$ ({\color{blue}blue} arrows in Figure \ref{fig:idea}).

\paragraph{5.1 Counterfactual Reasoning}
$(J_i(\pi))_{i=1}^n$ could have been directly estimated if we had access to $(M_i)_{i=1}^n$.
    However, how do we estimate $(J_i(\pi))_{i=1}^n$ when we only have $(H_i)_{i=1}^n$  collected using interactions via possibly different data collecting policies $(\beta_i)_{i=1}^n$?

To estimate $(J_i(\pi))_{i=1}^n$, we use the collected data $D_n$ and aim to answer the following counterfactual question: \textit{what would the performance of $\pi$ would have been, if $\pi$ was used to interact with $M_i$ instead of $\beta_i$}?
To answer this, we make the following standard support assumption \citep{thomas2015higha,thomas2016data,xie2019towards} that says that any action that is likely under $\pi$ is also sufficiently likely under the policy $\beta_i$ for all $i$.
%
%
\begin{ass}
    $\forall o \in \mathcal O, \forall a \in \mathcal A$, and $\forall i \leq n$,  $\frac{\pi(o,a)}{\beta_i(o,a)}$ 
    is bounded above by a (unknown) constant $c$.
    \thlabel{ass:support}
\end{ass}
Under \thref{ass:support}, an unbiased estimate of $J_i(\pi)$ can be obtained using common off-policy evaluation methods like importance sampling (IS) or per-decision importance sampling (PDIS) \citep{precup2000eligibility},
%
        $\forall i, \widehat J_i(\pi) \coloneqq \sum_{t=1}^T \rho_i^t R_i^t,  \text{    where, }\rho_i^t \coloneqq \prod_{j=1}^t\!\! \frac{\pi(O_i^j,A_i^j)}{\beta_i(O_i^j,A_i^j)}. \label{eqn:PDIS}$
%
This $\widehat J_i(\pi)$ provides an estimate of $J_i(\pi)$ associated with each $M_i$ and policy $\pi$, as needed for the {\color{red} red} arrows in Figure \ref{fig:idea}.
%

\paragraph{5.2 Double Counterfactual Reasoning}
%
%
Having obtained the estimates for $(J_i(\pi))_{i=1}^n$, we now aim to estimate how the performance of $\pi$ changes due to the underlying non-stationarity.
Recall that under active or hybrid non-stationarity, changes in a policy's performance due to the underlying non-stationarity is dependent on the past actions.
From \thref{ass:fixedf}, let
%
\begin{align}
    \forall i>0, \quad F_\pi(x, \pi', y) \coloneqq \Pr(J_{i+1}(\pi)=y|J_i(\pi)=x;\pi')
    \label{eqn:jk}
\end{align}
denote how the performance of $\pi$ changes between episodes, if $\pi'$ was executed. 
Here $J_{i+1}(\pi)$ is a random variable because of stochasticity in $H_i$ (i.e., how $\pi'$ interacts in $M_i$), as well as in the meta-transition from POMDP $M_i$ to $M_{i+1}$.
Similarly, let
%
\begin{align}
   \forall i>0, \quad f(J_i(\pi), \pi'; \theta_\pi) \coloneqq \mathbb{E}_{\pi'}\left[J_{i+1}(\pi) |J_i(\pi)\right] = \sum\nolimits_{y \in \mathbb R} F_\pi(J_i(\pi), \pi', y) y 
\end{align}
be some (unknown) function parameterized by $\theta_\pi \in \Theta$, which denotes the \textit{expected} performance of $\pi$ in episode $i+1$, if in episode $i$, $\pi$'s performance  was $J_i(\pi)$ and $\pi'$ was executed.
%
%
%
Parameters $\theta_\pi$ depend on $\pi$ and thus $f$ can model different types of changes to the performance of different policies.

Recall from Figure \ref{fig:idea} ({\color{blue}blue} arrows), if we can estimate $f(\cdot, \pi;\theta_\pi)$ to infer how $J_i(\pi)$ changes due to the underlying non-stationarity when interacting with $\pi$, then we can use it to predict $(J_i(\pi))_{i=n+1}^{n+L}$ when $\pi$ is deployed in the future.
In the following, we will predominantly focus on estimating $f(\cdot,\pi; \theta_\pi)$ using past data $D_n$.
Therefore, for brevity we let $f(\cdot;\theta_\pi) \coloneqq f(\cdot, \pi; \theta_\pi)$.

If pairs of $(J_i(\pi), J_{i+1}(\pi))$ are available when the transition between $M_i$ and $M_{i+1}$ occurs due to execution of $\pi$, then one could auto-regress $ J_{i+1}(\pi)$ on $J_i(\pi)$ to estimate $f(\cdot;\theta_\pi)$ and model the changes in the performance of $\pi$.
However, the sequence $(\widehat J_i(\pi))_{i=1}^n$ obtained from counterfactual reasoning cannot be used as-is for auto-regression.
This is because the changes that occurred between $M_i$ and $M_{i+1}$ are associated with the execution of $\beta_i$, not $\pi$.
For example, recall the toy robot example in Figure \ref{fig:badpassive}. If data was collected by mostly `running', then the performance of `walking' would decay as well.
Directly auto-regressing on the past performances of `walking' would result in how the performance of `walking' would change \textit{when actually executing `running'}.
However, if we want to predict performances of `walking' in the future, what we actually want to estimate is how the performance of `walking' changes \textit{if `walking' is actually performed}.

To resolve the above issue, we ask another counter-factual question:  \textit{What would the performance of $\pi$ in $M_{i+1}$ have been had we executed $\pi$, instead of $\beta_i$, in $M_i$?} 
%
%
In the following theorem we show how this question can be answered with a second application of the importance ratio $\rho_i \coloneqq \rho_i^T$. 
\begin{restatable}[]{rtheorem}{doubleISthm}
\thlabel{lemma:doubleIS}
Under \thref{ass:fixedf,ass:support},
$\forall m \in \mathcal M$ such that the performance $J(\pi)$ associated with $m$ is $j$, 
     $\mathbb{E}_{\pi}\left[J_{i+1}(\pi) |J_i(\pi)=j\right]
     = \mathbb{E}_{\beta_i,\beta_{i+1}}\big[\rho_i \widehat J_{i+1}(\pi) \big| M_{i}=m\big]$.
\end{restatable}
See Appendix \ref{apx:proofdis} for the proof.
Intuitively, as $\beta_i$ and $\beta_{i+1}$ were used to collect the data in $i$ and $i+1^\text{th}$ episodes, respectively, \thref{lemma:doubleIS} uses $\rho_i$ to first correct for the mismatch between  $\pi$ and  $\beta_i$ that influences how  $M_i$ changes to $M_{i+1}$ due to interactions $H_i$.
%
Secondly, $\widehat J_{i+1}$ corrects for the mismatch between  $\pi$ and $\beta_{i+1}$ for the sequence of interactions $H_{i+1}$ in $M_{i+1}$.

\paragraph{5.3 Importance-Weighted IV-Regression}
An important advantage of \thref{lemma:doubleIS} is that given $J_{i}(\pi)$, $\rho_i \widehat J_{i+1}(\pi)$ provides an unbiased estimate of $\mathbb{E}_{\pi}\left[J_{i+1}(\pi) |J_i(\pi)\right]$, even though $\pi$ may not have been used for data collection.
This permits using $Y_i \coloneqq \rho_i \widehat J_{i+1}(\pi)$ as a target for predicting the next performance given $X_i \coloneqq J_i(\pi)$, i.e., to estimate $f(J_i(\pi);\theta_\pi)$ through regression on $(X_i,Y_i)$ pairs.

However, notice that performing regression on the pairs $(X_i=J_{i}(\pi), Y_i=\rho \widehat J_{i+1}(\pi))_{i=1}^{n-1}$ may not be directly possible
 as we do not have $J_i(\pi)$; only unbiased \textit{estimates} $\widehat J_i(\pi)$ of $J_i(\pi)$.
This is problematic because in least-squares regression, while noisy estimates of the \textit{target} variable $Y_i$ are fine, noisy estimates of the \textit{input} variable $X_i$ may result in estimates of $\theta_\pi$ that are \textit{not even} asymptotically consistent \textit{even} when the underlying $f$ is a linear function of its inputs.
To see this clearly, consider the following naive estimator,
\begin{align}
        \hat \theta_\texttt{naive} &\in \argmin_{\theta \in \Theta} \,\, \sum\nolimits_{i=1}^{n-1}\left( f\left(\widehat J_{i}(\pi); \theta\right) -  \rho_{i}\widehat J_{i+1}(\pi) \right)^2. \label{eqn:bad}
\end{align}
Because $\widehat J_i(\pi)$ is an unbiased estimate of $J_{\pi}$, without loss of generality, let $\widehat J_i(\pi) = J_i(\pi) + \eta_i$, where $\eta_i$ is mean zero noise.
Let $\mathbb N \coloneqq [\eta_1, \eta_2, ..., \eta_{n-1}]^\top$ and $
    \mathbb J \coloneqq [J_1(\pi), J_2(\pi), ..., J_{n-1}(\pi)]^\top$.
 $\theta_\texttt{naive}$ can now be expressed as (see Appendix \ref{proof:bias}),
\begin{align}
    \hat \theta_\texttt{naive} \overset{a.s.}{\longrightarrow} \left(\mathbb J^\top \mathbb J + \mathbb N ^\top \mathbb N \right)^{-1} \mathbb{J}^\top \mathbb J \theta_\pi \,\, \cancel{\overset{a.s.}{\longrightarrow} } \, \theta_\pi. \label{eqn:naiveasymptote}    
\end{align}
Observe that $\mathbb N ^\top \mathbb N$ in \eqref{eqn:naiveasymptote} relates to the variances of the mean zero noise variables $\eta_i$.
The greater the variances, the more $\hat \theta_\texttt{naive}$ would be biased towards zero
(if $\forall i, \,\eta_i=0$, then the true $\theta_\pi$ is trivially recovered).
Intuitively, when the variance of $\eta_i$ is high, noise dominates and the structure in the data gets suppressed even in the large-sample regime.
Unfortunately, the importance sampling based estimator $\widehat J_i(\pi)$ in the sequential decision making setting is infamous for extremely high variance \citep{thomas2015higha}.
Therefore, $\hat \theta_\texttt{naive}$  can be extremely biased and will not be able to capture the trend in how performances are changing, \textit{even in the limit of infinite data and linear $f$}.
The problem may be exacerbated when $f$ is non-linear.

\paragraph{5.3.1 Bias Reduction}
To mitigate the bias stemming from noise in input variables, we introduce a novel instrument variable (IV)  \citep{pearl2000models} regression method for tackling non-stationarity.
Instrument variables $Z$ represent some side-information and were originally used in the causal literature to mitigate any bias resulting due to spurious correlation, caused by unobserved confounders, between the input and the target variables.
For mitigating bias in our setting, IVs can intuitively be considered as some side-information to `denoise' the input variable before performing  regression.
For this IV-regression, an ideal IV is \textit{correlated} with the input variables (e.g., $\widehat J_i(\pi)$) but \textit{uncorrelated} with the noises in the input variable (e.g., $\eta_i$).  
\todo{Mention error-in-variables}

%
We propose leveraging statistics based on past performances as an IV for $\widehat J_{i}(\pi)$.
For instance, using $Z_i \coloneqq \widehat J_{i-1}(\pi)$
 as an IV for $\widehat J_{i}(\pi)$.
Notice that while correlation between $J_{i-1}(\pi)$ and $J_{i}(\pi)$ can directly imply correlation between $\widehat J_{i-1}(\pi)$ and $\widehat J_{i}(\pi)$, values of  $J_{i-1}(\pi)$ and $J_{i}(\pi)$ are dependent on non-stationarity in the past.
Therefore, we make the following assumption, which may easily be satisfied when the consecutive performances do not change arbitrarily.
\begin{ass}
    $\forall i, \quad \operatorname{Cov}\big(\widehat J_{i-1}(\pi), \widehat J_{i}(\pi)\big) \neq 0$. 
    \thlabel{ass:correlated}
\end{ass}
However, notice that the noise in $\widehat J_{i}(\pi)$ can be \textit{dependent} on $\widehat J_{i-1}(\pi)$.
This is because non-stationarity can make $H_{i-1}$ and $H_{i}$ dependent, which are in turn used to estimate $\widehat J_{i-1}(\pi)$ and $\widehat J_{i}(\pi)$, respectively.
Nevertheless, perhaps interestingly, we show that despite not being independent,  $\widehat J_{i-1}(\pi)$ is \textit{uncorrelated} with the noise in $\widehat J_i(\pi)$.  
\begin{restatable}[]{rtheorem}{covthm} Under \thref{ass:fixedf,ass:support}, 
    $\forall i, \quad \operatorname{Cov}\big(\widehat J_{i-1}(\pi), \widehat J_{i}(\pi) - J_{i}(\pi)\big) = 0$.
    \thlabel{thm:cov}
\end{restatable}
See Appendix \ref{apx:proofiwiv} for the proof.
Finally, as IV regression requires learning an additional function $g\coloneqq \mathbb R \rightarrow \mathbb R$ parameterized by $\varphi \in \Omega$ (intuitively, think of this as a denoising function),  we let $\widehat J_{i-1}(\pi)$ be an IV for $\widehat J_{i}(\pi)$ and propose the following IV-regression based estimator, 
%
%
%
%
%
\begin{align}
    \hat \varphi_n &\in \argmin_{\varphi \in \Omega} \,\, \sum\nolimits_{i=2}^{n}\left( g\left(\widehat J_{i-1}(\pi); \varphi\right) - \widehat J_{i}(\pi) \right)^2 \label{eqn:g}
    \\
    \hat \theta_n &\in \argmin_{\theta \in \Theta}  \sum\nolimits_{i=2}^{n-1}\left( f\left(g\left(\widehat J_{i-1}(\pi); \hat \varphi_n\right); \theta\right) - \rho_{i}\widehat J_{i+1}(\pi) \right)^2\!\!\!. \label{eqn:g2}
\end{align}
%
%
%
%
\begin{restatable}[]{rtheorem}{consistencythm}
\thlabel{thm:consistent}
Under \thref{ass:fixedf,ass:support,ass:correlated}, if $f$ and $g$ are linear functions of their inputs, then $\hat \theta_n$ is a strongly consistent estimator of $\theta_\pi$, i.e., 
%
    $\hat \theta_n 
    \overset{\text{a.s.}}{\longrightarrow}
    \theta_\pi$.
%
(See Appendix \ref{apx:proofiwiv} for the proof.)
\end{restatable} 
%
%
\begin{rem}
    Other choices of instrument variables $Z_i$ (apart from $Z_i = \widehat J_{i-1}(\pi)$) are also viable.
    We discuss some alternate choices in Appendix
    \ref{apx:arp}.
    These other IVs can be used in \eqref{eqn:g} and \eqref{eqn:g2} by replacing $\widehat J_{i-1}(\pi)$ with the alternative $Z_i$.
\end{rem}

\begin{rem}
    As discussed earlier, it may be beneficial to model $J_{i+1}(\pi)$ using $(J_k(\pi))_{k=i-p+1}^{i}$ with $p>1$.
    The proposed estimator can be easily extended by making  $f$ dependent on multiple past terms
    $(X_k)_{k=i-p+1}^i$, where $\forall k, \, X_k \coloneqq g((\widehat J_l(\pi))_{l=k-p}^{k-1}; \hat \phi)$.
    We discuss this in more detail in Appendix \ref{apx:arp}.
    The proposed procedure is also related to methods that use lags of the time series as instrument variables \citep{bellemare2017lagged, wilkins2018lag, wang2019lagged}.
    \thlabel{rem:p}
\end{rem}

\begin{rem}
    An advantage of the model-free setting is that we only need to consider changes in $J(\pi)$, which is a \emph{\textbf{scalar}} statistic.
    For scalar quantities, \emph{linear} auto-regressive models have been known to be useful in modeling a wide variety of time-series trends.
    %
    Nonetheless, \emph{non-linear} functions like RNNs and LSTMs \citep{hochreiter1997long} may also be leveraged using deep instrument variable methods \citep{hartford2017deep,bennett2019deep,liu2020deep,xu2020learning}.   %
\end{rem}

As required for the \textcolor{blue}{blue} arrows in Figure \ref{fig:idea}, $f(\cdot; \hat \theta_n)$ can now be used to estimate the expected value $\mathbb{E}_{\pi}\left[J_{i+1}(\pi) |J_i(\pi)\right]$ under hybrid non-stationarity.
Therefore, using $f(\cdot; \hat \theta_n)$ we can now auto-regressively forecast the future values of $(J_i(\pi))_{i=n+1}^{n+L}$ and obtain an estimate for $\mathscr J(\pi)$.
%
%
A complete algorithm for the proposed procedure is provided in Appendix \ref{apx:algo}.

\paragraph{5.3.2 Variance Reduction}
As discussed earlier, importance sampling results in noisy estimates of $J_{i}(\pi)$.
During regression, while high noise in the input variable leads to high bias, high noise in the target variables leads to high variance parameter estimates.
Unfortunately, \eqref{eqn:g} and \eqref{eqn:g2} have target variables containing $\rho_i$ (and $\rho_{i+1}$) which depend on the product of importance ratios and can thus result in extremely large values leading to higher variance parameter estimates.

The instrument variable technique helped in mitigating bias.
To mitigate variance, we draw inspiration from the reformulation of weighted-importance sampling presented for the \textit{stationary} setting by \citet{mahmood2014weighted}, and propose the following estimator,
\begin{align}
    \tilde \varphi_n &\in \argmin_{\varphi \in \Omega} \,\, \sum_{i=2}^{n} \bar \rho_i \left( g\left(\widehat J_{i-1}(\pi); \varphi\right) -  G_{i}(\pi) \right)^2, & \text{where \,\,}  \bar \rho_i \coloneqq \frac{\rho_{i}}{\sum_{j=2}^{n}\rho_j}
    \label{eqn:gaa} 
    \\
    \tilde \theta_n &\in \argmin_{\theta \in \Theta}  \sum_{i=2}^{n-1} \rho_i^\dagger \left( f\left(g\left(\widehat J_{i-1}(\pi); \tilde \varphi_n\right); \theta\right) -  G_{i+1}(\pi) \right)^2, & \text{where \,\,} \rho_i^\dagger \coloneqq \frac{\rho_{i}\rho_{i+1}}{\sum_{j=2}^{n-1}\rho_j \rho_{j+1}} \label{eqn:s2aa}
\end{align}
where $G_{i}$ is the return observed for $M_{i}$.
Intuitively, instead of importance weighting the \textit{target},  we importance weight the squared error, proportional to how likely that \textit{error} would be if $\pi$ was used to collect the data.
Since dividing by any constant does not affect $\tilde \varphi_n$ and $\tilde \theta_n$, the choice of $\bar \rho_i$ and $\rho_i^\dagger$ ensures that both $\bar \rho_i$ and $\rho_i^\dagger \in [0,1]$, thereby mitigating variance but still providing consistency.
\begin{restatable}[]{rtheorem}{wconsistencythm}

\thlabel{thm:wconsistent}
Under \thref{ass:fixedf,ass:support,ass:correlated}, if $f$ and $g$ are linear functions of their inputs, then $\tilde \theta_n$ is a strongly consistent estimator of $\theta_\pi$, i.e.,
%
    $\tilde \theta_n 
    \overset{\text{a.s.}}{\longrightarrow}
    \theta_\pi$.
%
(See Appendix \ref{apx:proofiwiv} for the proof.)
\end{restatable} 

\section{Empirical Analysis}
\begin{figure}[t]
    \centering
    \centering
    \includegraphics[width=0.32\textwidth]{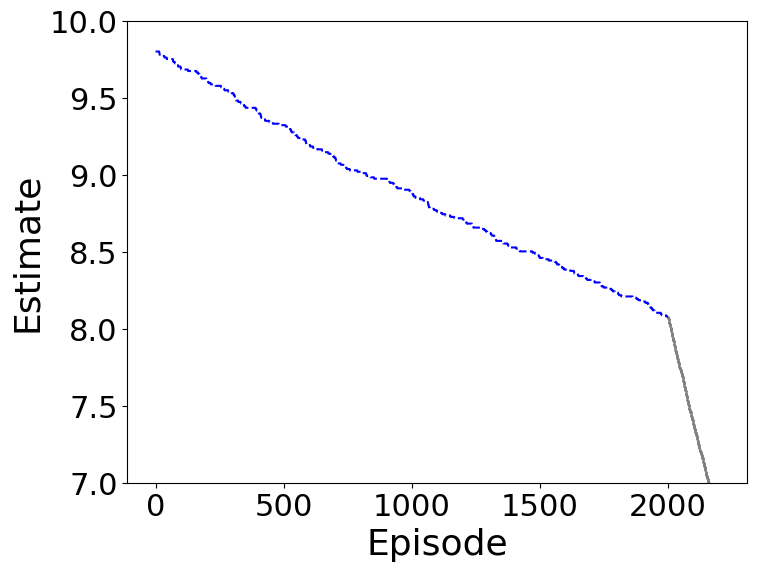}
    \includegraphics[width=0.32\textwidth]{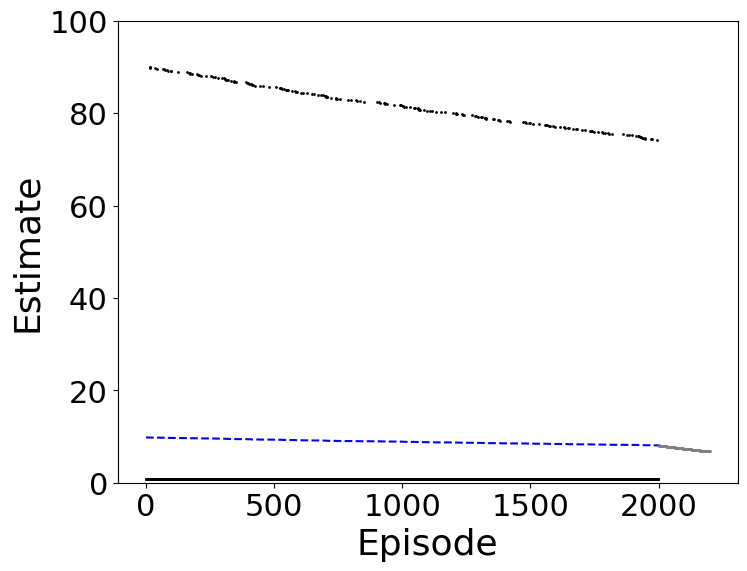}
    \includegraphics[width=0.32\textwidth]{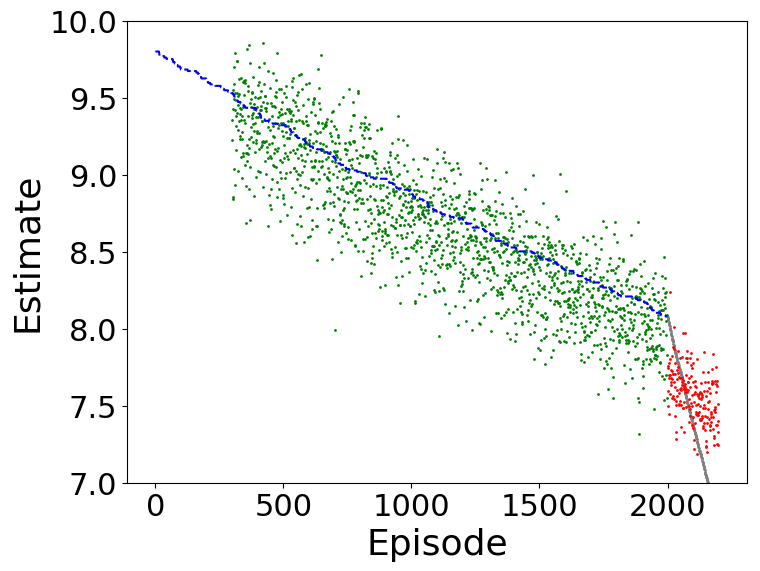}
    \\
    \includegraphics[width=0.9\textwidth]{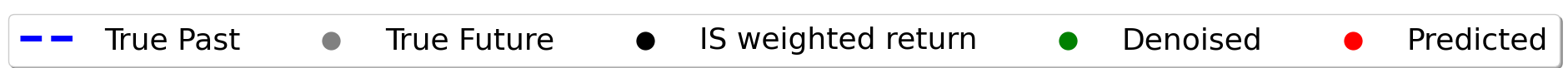}
    \caption{
An illustration of the stages in the proposed method for the RoboToy domain of Figure \ref{fig:badpassive}.
Here, evaluation policy $\pi$ chooses to `run' more often, whereas the data collecting  policy $\beta$ chooses to `walk' more often.
%
\textbf{(Left)} This results in a slow decline of performance for $\pi$ initially, followed by a faster decline once $\pi$ is deployed after episode $2000$. 
The blue and gray curves are unknown to the algorithm.
\textbf{(Middle)} OPEN first uses historical data to obtain counterfactual estimates of $J_i(\pi)$ for the past episodes.
    One can see the high-variance in these estimates \textbf{(notice the change in the y-scale)} due to the use of importance sampling.
\textbf{(Right)} Intuitively, before naively auto-regressing, OPEN first denoises past performance estimates using the first stage of IV regression (i.e., converts black dots to green dots). 
It can be observed that OPEN successfully denoises the importance sampling estimates.
Using these denoised estimates and a second use of counterfactual reasoning, OPEN performs the second stage of IV regression.
It is able to estimate that once $\pi$ is deployed, performances in the future will decrease more rapidly compared to what was observed in the past.
}
    \label{fig:activestep}
\end{figure}

This section presents both qualitative and quantitative empirical evaluations using several environments inspired by real-world applications that exhibit non-stationarity. 
In the following paragraphs, we first
briefly discuss different algorithms being compared and answer three primary questions.\footnote{Code is available at \href{https://github.com/yashchandak/activeNS}{https://github.com/yashchandak/activeNS}}

\textbf{1. OPEN: } We call our proposed method OPEN: \underline{o}ff-\underline{p}olicy \underline{e}valuation for \underline{n}on-stationary domains with structured passive, active, or hybrid changes.
It is based on our bias and variance reduced estimator developed in \eqref{eqn:gaa} and \eqref{eqn:s2aa}.
Appendix \ref{apx:algo} contains the complete algorithm.

\textbf{2. Pro-WLS: } For the baseline, we use Prognosticator with weighted least-squares (Pro-WLS) \citep{chandak2020optimizing}.
This method is designed to tackle only passive non-stationarity.

%

\textbf{3. WIS: } A weighted importance sampling based estimator that ignores presence of non-stationarity completely \citep{precup2000eligibility}.

\textbf{4. SWIS: } Sliding window extension of WIS which instead of considering all the data, only considers data from the recent past.
%


\clearpage
\textbf{\textit{Q1. (Qualitative Results) What is the impact of the two stages of the OPEN algorithm?}}

 In Figure \ref{fig:activestep} we present a step by step breakdown of the intermediate stages of a single run of OPEN on the RoboToy domain from Figure \ref{fig:badpassive}.
%
%
It can be observed that OPEN is able to extract the effect of the underlying active non-stationarity on the performances and also detect that the evaluation policy $\pi$ that `runs'  more often will cause an active harm, if deployed in the future.

\textbf{\textit{Q2. (Quantitative Results) What is the effect of different types and rates of non-stationarity?}}
\begin{figure}[t]
    \centering
    \includegraphics[width=0.24\textwidth]{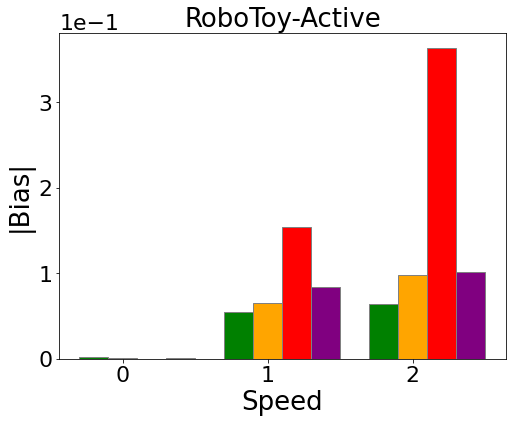}
    \includegraphics[width=0.24\textwidth]{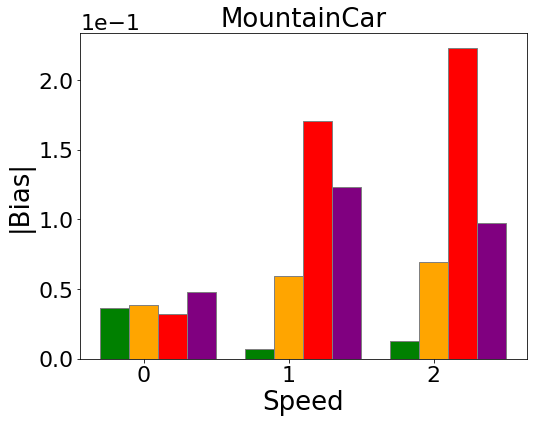}
    \includegraphics[width=0.24\textwidth]{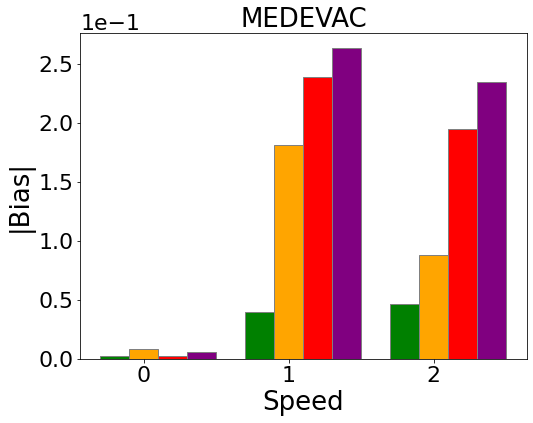}
    \includegraphics[width=0.24\textwidth]{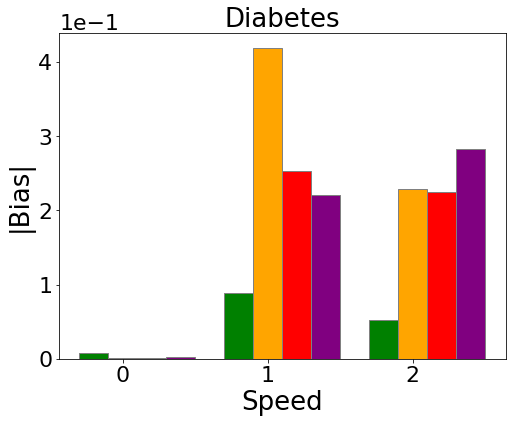}
    \\
    \includegraphics[width=0.24\textwidth]{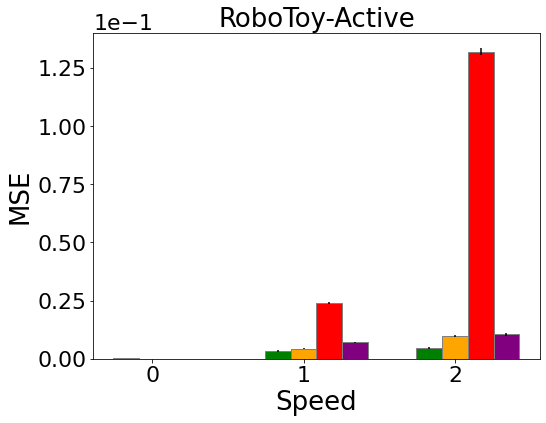}
    \includegraphics[width=0.24\textwidth]{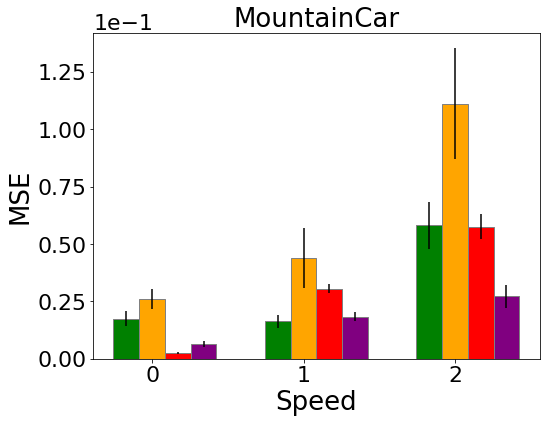}
    \includegraphics[width=0.24\textwidth]{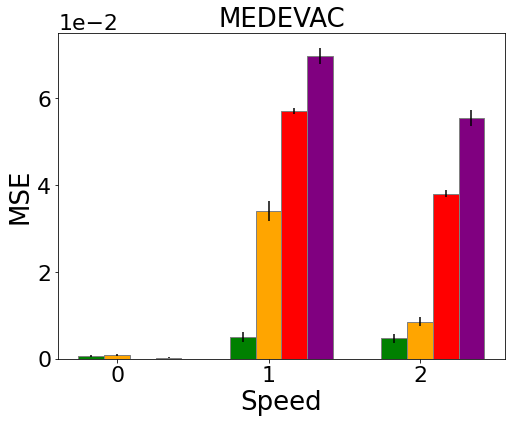}
    \includegraphics[width=0.24\textwidth]{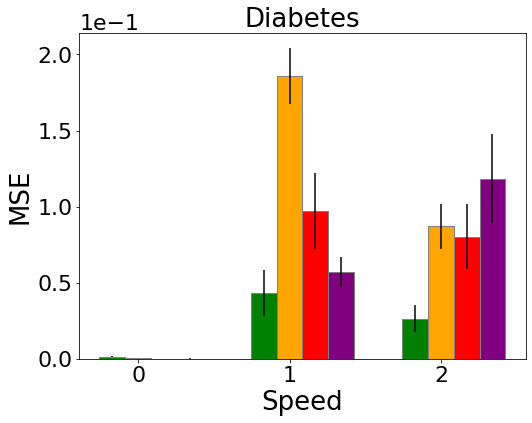}
    \\
    \includegraphics[width=0.55\textwidth]{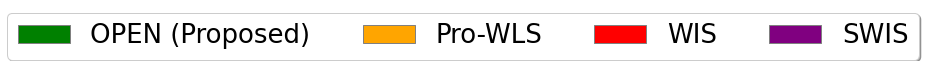} 
      \caption{Comparison of different algorithms for predicting the future performance of evaluation policy $\pi$ on domains that exhibit active/hybrid non-stationarity.
      On the x-axis is the speed which corresponds to the rate of non-stationarity; higher speed indicates faster rate of change and a speed of zero indicates stationary domain. 
      On the y-axes are the absolute bias \textbf{(Top row)} and the mean-squared error \textbf{(Bottom row)} of the predicted performance estimate \textit{(lower is better everywhere)}.
      For each domain, for each speed, for each algorithm, 30 trials were executed.
    }
    \label{fig:active_plots_bias}
\end{figure}

Besides the toy robot from Figure \ref{fig:badpassive}, we provide empirical results on three other domains inspired by real-world applications that exhibit non-stationarity.
Appendix \ref{sec:envdet} contains details for each, including how the evaluation policy and the data collecting policy were designed for them.

\textbf{Non-stationary Mountain Car: }
 In real-world mechanical systems, motors undergo wear and tear over time based on how vigorously they have been used in the past.
 To simulate similar performance degradation, we adapt the classic (stationary) mountain car domain \citep{SuttonBarto2}.
  We modify the domain such that after every episode the effective acceleration force is decayed proportional to the average velocity of the car in the current episode.
 This results in active non-stationarity, where the change in the system is based on the actions taken by the agent in the past.

\textbf{Type-1 Diabetes Management: }
Personalised automated healthcare systems for individual patients should account for the physiological and lifestyle changes of the patient over time.
To simulate such a scenario we use an open-source implementation \citep{simglucose} of the U.S. Food and Drug Administration (FDA) approved Type-1 Diabetes Mellitus simulator (T1DMS) \citep{man2014uva} for the treatment of Type-1 diabetes,
where we induced non-stationarity by oscillating the body parameters (e.g., rate of glucose absorption, insulin sensitivity, etc.) between two known configurations available in the simulator. 
This induces passive non-stationarity, that is, changes are not dependent on past actions.

\textbf{MEDEVAC: } This domain stands for \underline{med}ical \underline{evac}uation using air ambulances. This domain was developed by \citet{robbins2020approximate} for optimally routing air ambulances to provide medical assistance in regions of conflict.
Based on real-data, this domain simulates the arrival of different events, from different zones, where each event can have different priority levels.
Serving higher priority events yields higher rewards.
A good controller decides whether to deploy, and which MEDEVAC to deploy, to serve any event (at the risk of not being able to serve a new high-priority event if all ambulances become occupied).
%
%
Here, the arrival rates of different events can change based on external incidents during conflict. Similarly, the service completion rate can also change based on how frequently an ambulance is deployed in the past.
To simulate such non-stationarity, we oscillate the arrival rate of the incoming high-priority events, which induces passive non-stationarity.
Further, to induce wear and tear, we decay the service rate of an ambulance proportional to how frequently the ambulance was used in the past.
This induces active non-stationarity. 
The presence of both active and passive changes makes this domain subject to hybrid non-stationarity.

\label{sec:actbias}

Figure \ref{fig:active_plots_bias}  presents the (absolute) bias and MSE incurred by different algorithms for predicting the future performance of the evaluation policy $\pi$.
As expected, the baseline method WIS that ignores the non-stationarity completely fails to capture the change in performances over time.
Therefore, while WIS works well for the stationary setting, as the rate of non-stationarity increases, the bias incurred by WIS grows.
In comparison, the baseline method Pro-WLS that can only account for passive non-stationarity captures the trend better than WIS, but still performs poorly in comparison to the proposed method OPEN that is explicitly designed to handle active/hybrid non-stationarity.
Perhaps interestingly, for the Diabetes domain which only has passive non-stationarity, we observe that OPEN performs better than Pro-WLS.
As we discuss later, this can be attributed to the sensitivity of Pro-WLS to its hyper-parameters.


While OPEN incorporated one variance reduction technique,  it can be noticed when the rate of non-stationarity is high, variance can sometimes still be high thereby leading to higher MSE.
%
%
We discuss potential negative impacts of this in Appendix \ref{sec:FAQ}. %
Incorporating (partial) knowledge of the underlying model and developing doubly-robust version of OPEN could potentially mitigate variance further.
We leave this extension for future work.

\textbf{\textit{Q3. (Ablations Results) How robust are the methods to hyper-parameters?}}

Due to space constraints, we defer the empirical results and discussion for this to Appendix \ref{apx:ablation}.
Overall, we observe that the proposed method OPEN being an auto-regressive method can extrapolate/forecast better and is thus more robust to hyper-parameters (number of past terms to condition, as discussed in \thref{rem:p}) than Pro-WLS that uses Fourier bases for regression (where the hyper-parameter is the order of Fourier basis) and is not as good for extrapolation.

\section{Conclusion}
%
We took the first steps for addressing the fundamental question of off-policy evaluation under the presence of non-stationarity.
Towards this goal we discussed the need for structural assumptions and developed a model-free procedure OPEN and presented ways to mitigate its bias and variance.
Empirical results suggests that OPEN can now not only enable practitioners to predict future performances amidst non-stationarity but also identify  policies that may be actively causing harm or damage. 
In the future, OPEN can also be extended to enable \textit{control} of non-stationary processes.
%
%

\todo{Limitations}

\section{Acknowledgements}

Research reported in this paper was sponsored in part by a gift from Adobe, NSF award \#2018372.
This work was also funded in part by the U.S. Army Combat Capabilities Development Command (DEVCOM) Army Research Laboratory under Cooperative Agreement W911NF-17-2-0196 and Support Agreement No. USMA21050. The views expressed in this paper are those of the authors and do not reflect the official policy or position of the United States Military Academy, the United States Army, the Department of Defense, or the United States Government. The U.S.~Government is authorized to reproduce and distribute reprints for Government purposes notwithstanding any copyright notation herein.

\bibliography{mybib}
\bibliographystyle{abbrvnat}

\section*{Checklist}


\begin{enumerate}

\item For all authors...
\begin{enumerate}
  \item Do the main claims made in the abstract and introduction accurately reflect the paper's contributions and scope?
    \answerYes{}
  \item Did you describe the limitations of your work?
    \answerYes{}
  \item Did you discuss any potential negative societal impacts of your work?
    \answerYes{}
  \item Have you read the ethics review guidelines and ensured that your paper conforms to them?
    \answerYes{}
\end{enumerate}

\item If you are including theoretical results...
\begin{enumerate}
  \item Did you state the full set of assumptions of all theoretical results?
    \answerYes{}
        \item Did you include complete proofs of all theoretical results?
    \answerYes{}
\end{enumerate}

\item If you ran experiments...
\begin{enumerate}
  \item Did you include the code, data, and instructions needed to reproduce the main experimental results (either in the supplemental material or as a URL)?
    \answerYes{}
  \item Did you specify all the training details (e.g., data splits, hyperparameters, how they were chosen)?
    \answerYes{}
        \item Did you report error bars (e.g., with respect to the random seed after running experiments multiple times)?
    \answerYes{}
        \item Did you include the total amount of compute and the type of resources used (e.g., type of GPUs, internal cluster, or cloud provider)?
    \answerYes{}
\end{enumerate}

\item If you are using existing assets (e.g., code, data, models) or curating/releasing new assets...
\begin{enumerate}
  \item If your work uses existing assets, did you cite the creators?
    \answerYes{}
  \item Did you mention the license of the assets?
    \answerNA{}
  \item Did you include any new assets either in the supplemental material or as a URL?
    \answerNA{}
  \item Did you discuss whether and how consent was obtained from people whose data you're using/curating?
    \answerNA{}
  \item Did you discuss whether the data you are using/curating contains personally identifiable information or offensive content?
    \answerNA{}
\end{enumerate}

\item If you used crowdsourcing or conducted research with human subjects...
\begin{enumerate}
  \item Did you include the full text of instructions given to participants and screenshots, if applicable?
    \answerNA{}
  \item Did you describe any potential participant risks, with links to Institutional Review Board (IRB) approvals, if applicable?
    \answerNA{}
  \item Did you include the estimated hourly wage paid to participants and the total amount spent on participant compensation?
    \answerNA{}
\end{enumerate}

\end{enumerate}

\clearpage
\include{appendix}


\end{document}

%% file: header.tex
\usepackage{amsmath}
\usepackage{amsthm}
\usepackage{amssymb}
\usepackage{bbm}
\usepackage{mathtools}
\usepackage{mathrsfs}
\usepackage{dsfont}

\usepackage{courier} 

\usepackage{lipsum} 

\def\mul#1#2{\color{#1}\underline{{\color{black}#2}}\color{black}}

\usepackage[space]{grffile} 
\usepackage{cancel} 

\usepackage{hyperref}

\usepackage{graphicx}
\usepackage{wrapfig}
\usepackage{xcolor}
\definecolor{dark-red}{rgb}{0.4,0.15,0.15}
\definecolor{dark-blue}{rgb}{0,0,0.7}
\hypersetup{
    colorlinks, linkcolor={dark-blue},
    citecolor={dark-blue}, urlcolor={dark-blue}
}

\DeclareMathOperator*{\argmin}{argmin}

\usepackage{multicol}
\usepackage{enumitem}
\usepackage[utf8]{inputenc} 
\usepackage[T1]{fontenc}    
\usepackage{hyperref}       
\usepackage{url}            
\usepackage{booktabs}       
\usepackage{nicefrac}       
\usepackage{bm}

\usepackage{etoc}

\usepackage{theoremref}
\mathtoolsset{showonlyrefs}

\newtheorem{rem}{Remark}[]

\newtheorem{ass}{Assumption}[]

\usepackage{thmtools}
\usepackage{thm-restate}

\DeclarePairedDelimiter\autobracket{(}{)}
\newcommand{\br}[1]{\autobracket*{#1}}

\DeclarePairedDelimiter\autobrackett{[}{]}
\newcommand{\brr}[1]{\autobrackett*{#1}}

%% file: appendix.tex
\onecolumn
\setcounter{thm}{0}

\appendix

\begin{center}
    \Large
    \textbf{Off-Policy Evaluation for Action-Dependent \\ Non-Stationary Environments \\
    (Appendix)}
\end{center}

\etocdepthtag.toc{mtappendix}
\etocsettagdepth{mtchapter}{none}
\etocsettagdepth{mtappendix}{subsection}
\tableofcontents

\section{FAQs: Frequently Asked Questions}
\label{sec:FAQ}
    \subsection{How does the stationarity condition for a time-series differ from that in RL?\\
    }
    
    Conventionally, stationarity is the time-series literature refers to the condition where the distribution (or few moments) of a finite sub-sequence of random-variables in a time-series remains the same as we shift it along the time index axis \citep{cox2017theory}.
    In contrast, the stationarity condition in the RL setting implies that the environment is fixed \citep{SuttonBarto2}. This makes the performance $J(\pi)$ of any policy $\pi$ to be a constant value throughout.
    In this work, we use `stationarity' as used in the RL literature.
    \subsection{Can the POMDP during each episode (Figure \ref{fig:controlgraph}) itself be non-stationary?\\}

    Any source of non-stationarity can be incorporated in the (unobserved) state to induce another stationary POMDP (from which we can obtain a single sequence of interaction).
    The key step towards tractability is \thref{ass:fixedf} that enforces additional structure on the performance of any policy across the \textit{sequence} of (non-)stationary POMDPs. 

    \subsection{What if it is known ahead of time that the non-stationarity is passive only?\\}
    In such cases where the underlying changes are independent of the past actions, $\mathbb{E}_{\beta_1}[J_{i+1}(\pi)|J_{i}(\pi)] = \mathbb{E}_{\beta_2}[J_{i+1}(\pi)|J_{i}(\pi)]$, for any policies $\beta_1$ and $\beta_2$. 
    Therefore, there is no need for double-counterfactual reasoning to correct for the changes observed in the past.
    Particularly, in \thref{lemma:doubleIS} the second use of importance sampling can be avoided as $\mathbb{E}_{\beta_i,\beta_{i+1}}\left[\rho_i \widehat J_{i+1}(\pi) \middle| M_{i}(\pi)\right] = \mathbb{E}_{\beta_i,\beta_{i+1}}\left[ \widehat J_{i+1}(\pi) \middle| M_{i}(\pi)\right]$ under passive non-stationarity.
    Rest of the procedure for OPEN can be modified accordingly.
    
    \subsection{How should different non-stationarities be treated in the on-policy setting?\\}
    Perhaps interestingly, OPEN makes no effective distinction between active and passive non-stationarity in the on-policy setting.
    Notice that in the on-policy setting, importance ratios $\rho=1$ everywhere, therefore the use of double counterfactual reasoning has no impact.
    Intuitively, in the on-policy setting, there is no need to dis-entagle the active and passive sources of non-stationarity, as the prediction needs to be made about the same policy that was used during data collection.
    \subsection{Can you tell us more about when would \thref{ass:fixedf} be (in)valid?\\}
    Yes, we provide a detailed discussion on \thref{ass:fixedf} in Appendix \ref{sec:assumption}.

    \subsection{What are the limitations and potential negative impacts of the work?\\}
    Our work presents the first few steps towards off-policy evaluation in the presence of non-stationarity.
    Towards this goal, we used \thref{ass:fixedf} to enforce a higher-order stationarity condition. 
    We have provided extended discussion regarding the same in Appendix \ref{sec:assumption} and a practitioner should carefully analyze their problem setup to conclude if the assumption holds (at least approximately).

    Further, often off-policy evaluation is used in safety-critical settings, where it is important to provide confidence intervals \citep{thomas2015higha,thomas2019preventing,jiang2020minimax}. 
    Because of our use of instrument variables, our estimator may have high-variance.
    This can be explained by observing the closed form equation in \eqref{eqn:simpleIV2} obtained using the IV procedure. Here, $Z$ is the instrument variable and if it is weakly correlated with X (i.e,. $Z^\top X$ has a small magnitude) then $(Z^\top X)^{-1}$ can be large thereby increasing variance.
    However, our proposed method OPEN only provides point-estimates and thus using it as-is in safety critical settings would be irresponsible.

    If the application does exhibit non-stationarity, a practitioner may have to make a tough choice between prior methods that provide confidence intervals under the stationarity assumption, or the proposed method that may be applicable to their non-stationary  setting but does not provide any confidence intervals.

\section{Extended Related Work}
\label{sec:related}

In this section we discuss several different research directions that are relevant to the topic of this paper.
We refer the readers to the work by \citet{padakandla2020survey, khetarpal2020towards} for a more exhaustive survey.

\subsection{Off-policy evaluation in stationary domains}

In the off-policy RL setup, there is a large body of literature that tackles the off-policy estimation problem.
%
%
%
One line of work leverages dynamic programming \citep{puterman1990markov,SuttonBarto2} to develop off-policy estimators \citep{boyan1999least,sutton2008convergent,sutton2009fast,mahmood2014weighted,mahmood2015emphatic}.
Several recent approaches also build upon a dual perspective for dynamic programming \citep{puterman1990markov,wang2007dual,nachum2020reinforcement} for performing off-policy evaluation \citep{liu2018breaking, xie2019towards,jiang2020minimax,uehara2020minimax,dai2020coindice,feng2021nonasymptotic}. 
These works require fully-observable states.
Other direction of work takes Monte-Carlo perspective to perform trajectory based importance sampling and are applicable to stationary setting with partial observability \citep{precup2000eligibility,thomas2015higha,jiang2015doubly,thomas2016data}.
The proposed work builds upon this direction.

Several works have also discussed various techniques for variance reduction \citep{jiang2015doubly, thomas2016data,munos2016safe,harutyunyan2016q,liu2018breaking,espeholt2018impala, nachum2019dualdice,yang2020off,yuan2021sope}. 
However, these methods are restricted to stationary domains.
%

\subsection{Non-stationarity in stationary domains}

In the face of uncertainty, prior works often opt for exploratory or safe behavior by acting optimistically or pessimistically, respectively.
This is often achieved by using the collected data to dynamically modify the observed rewards for any state-action pair by either providing bonuses \citep{agarwal2020pc,taiga2021bonus} or penalties \citep{buckman2020importance,cetin2021learning}.
One could view this as an instance of active non-stationarity.
Similarly, in temporal-difference (TD) methods the target for the value function keeps changing and such changes are also dependent on the data collected in the past \citep{SuttonBarto2}. 
However, we note that such non-stationarities are only artifacts of the learning algorithm as the underlying domain remains stationary throughout.
In contrast, the focus of our work is on settings where the underlying domain is non-stationary.

\subsection{Single Episode Continuing setting} 
As discussed in Section \ref{sec:ass}, non-stationarity  can be alternatively modeled using a single long episode in a stationary POMDP.
From this point of view, one may wonder if the average-reward/continuing setting \citep{SuttonBarto2} could be useful?
While there have been off-policy evaluation methods designed to tackle the continuing setting \citep{liu2018breaking,nachum2019dualdice,yang2020off}, they require two important conditions that are no applicable for our setting:
\textbf{(a)} They assume access to the true underlying state such that there is no partial-observability, and (b) They assume that the transition tuples are sampled from the stationary state-visitation distribution of a policy.
In the non-stationary setting that we consider, we may not have data from any stationary state visitation distribution, and we may not have access to the true underlying states either.
%
%
%

\subsection{Non-stationarity in MDPs/Bandits}

Several prior methods have considered tackling non-stationarity for reinforcement learning problems.
For instance, a Hidden-Mode MDP is a setting that assumes that the environment changes are confined to a few hidden modes, where each mode represents a unique MDP. 
This provides a tractable way to model a limited number of MDPs 
\citep{choi2000environment,basso2009reinforcement},
or perform updates using mode-change detection  \citep{da2006dealing,padakandla2019reinforcement, alegre2021minimum}.
Similarly there are methods \citep{xie2020deep} based on hidden-parameter MDPs \citep{doshi2016hidden} that consider a more general setup where the hidden variable can be continuous.
Alternatively, many methods \citep{thomas2017predictive,jagerman2019when, chandak2020optimizing,zhou2020nonstationary,poiani2021meta,liotet2021lifelong} have considered time-dependent MDPs \citep{rachelson2009timdppoly}.
Aspects related to safety and confidence intervals have also been explored  \citep{ammar2015safe,chandak2020towards,chandak2021universal}.
However, the focus of these methods are on settings with passive non-stationarity, where the past actions do not influence the underlying non-stationarity.
Our works extends this direction of research to provide off-policy evaluation amidst active and hybrid non-stationarity as well.

Non-stationary multi-armed bandits (NMAB) capture the setting where the horizon length is one, but the reward distribution changes over time \citep{moulines2008,besbes2014stochastic,russac2019weighted,vernade2020non}.
%
%
%
Many variants of NMAB, like \textit{cascading non-stationary bandits} \citep{wang2019aware,li2019cascading} and  \textit{rotting bandits} \citep{levine2017rotting,seznec2018rotting} have also been considered.
In contrast, this work focuses on methods that generalize to the sequential decision making setup where the horizon length can be more than 1. 

\subsection{Multi-agent Games } 
Non-stationarity also occurs in multiplayer games \citep{singh2000nash,bowling2005convergence,conitzer2007awesome} where the opponent can change their strategy as a response to the agent's previous decisions.
These types of changes are related to active non-stationarity that we consider in this work. 
In such games, opponent modeling has been shown to be useful and regret bounds for multi-player games \citep{zhang2010multi,mealing2013opponent,foster2016learning, foerster2018learning}.
%
%
%
%
%
%
Further, often these games still assume that the underlying system/environment (excluding other players) is stationary and focus on searching for (Nash) equilibria.
Similarly, non-stationarities are also induced in the multi-agent systems where an agent tries to influence other agents \citep{jaques2019social,wang2019influence, xie2020learning,wang2021influencing}.
However, under general non-stationarity, the underlying system may also change and thus there may not even exist any fixed equilibria.
Perhaps a more relevant setting would be that of evolutionary/dynamics games, where the pay-off matrix and specification of the game can change over time \citep{gemp2017online, hennes2019neural}.
%
%
Such methods, however, do not leverage any underlying structure in how the game is changing nor do they account for settings where the changes might be a consequence of past interactions of the agent.
While relevant, these other research areas are distinct from our setting of interest.

\subsection{Dynamical Systems and Time-Series Analysis}

The proposed method for modeling the evolution of a policy's performance over time using stochastic estimates of past performances may be reminiscent of state-space methods (e.g., Kalman filtering) for dynamical systems \citep{hamilton1994state}.
However, in comparison to these methods,  we do not need to model noise variables, which could have been challenging in our case as noise is heteroskedastic because of past (off-policy) performance estimates being computed using data from different behavior policies. 
Further the form of OPEN estimator allows leveraging (accelerated) gradient descent based optimizers to obtain the solution instead of relying on computationally expensive closed-form solutions that are typically needed by state-space models. Due to this, in practice our method can also be used with non-linear functions $f$ (e.g., recurrent neural network based auto-regressive models). 

Different applications of time-series analysis have also discussed the use of lags as instruments \citep{achen2000lagged,  reed2015practice,bellemare2017lagged, wilkins2018lag, wang2019lagged}.
Our use case differs from these prior works in that we look at the full sequential decision making setup for reinforcement learning, and also consider a novel importance-weighted instrument-variable regression model.

\section{Discussion on the Structural Assumption}

\label{sec:assumption}

\thref{ass:fixedf} states that  $\forall m \in \mathcal M$ such that the performance $J(\pi)$ associated with $m$ is $j$, 
    \begin{align}
        \forall i,\, \Pr(J_{i+1}(\pi)=j_{i+1}| M_i=m; \pi) = \Pr(J_{i+1}(\pi)=j_{i+1}|J_i(\pi)=j; \pi). 
    \end{align}

As discussed earlier, consider a `meta-transition' function that characterizes $\Pr(J_{i+1}(\pi)|J_i(\pi),\pi')$ similar to how the standard transition function in an MDP characterizes $\Pr(S_{t+1}|S_t,A_t)$.
This assumption is imposing the following two conditions: \textbf{(a)} A \textit{higher-order stationarity} condition on the meta-transitions under which non-stationarity can result in changes over time, but \textit{the way the changes happen is fixed}, 
and \textbf{(b)} Knowing the past performance(s) of a policy $\pi$ provides \textit{sufficient} information for the meta-transition function to model how the performance will change upon executing any (possibly different) policy $\pi'$. 
We provide some examples in Figure \ref{fig:assex} to demonstrate few settings to discuss the applicability of this assumption.

\begin{figure}[h]
    \centering
    \includegraphics[width=0.3\textwidth]{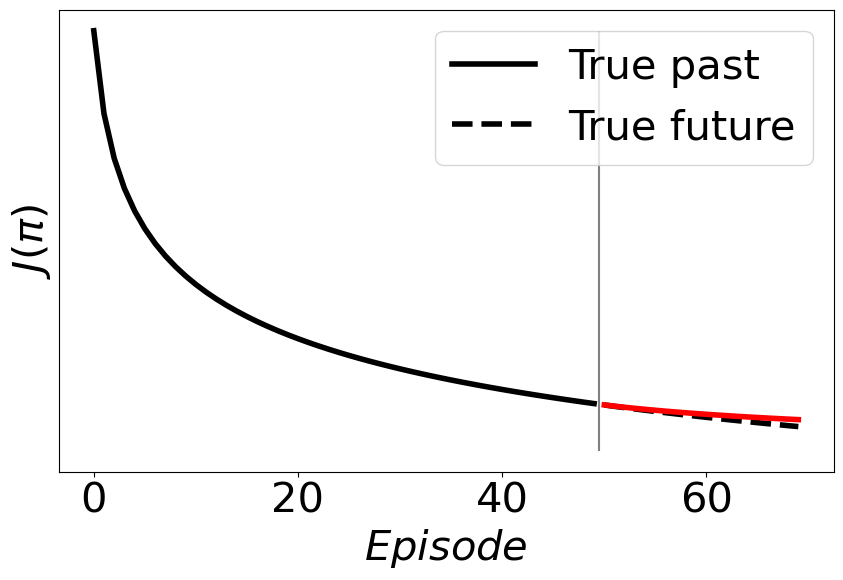}
    \includegraphics[width=0.3\textwidth]{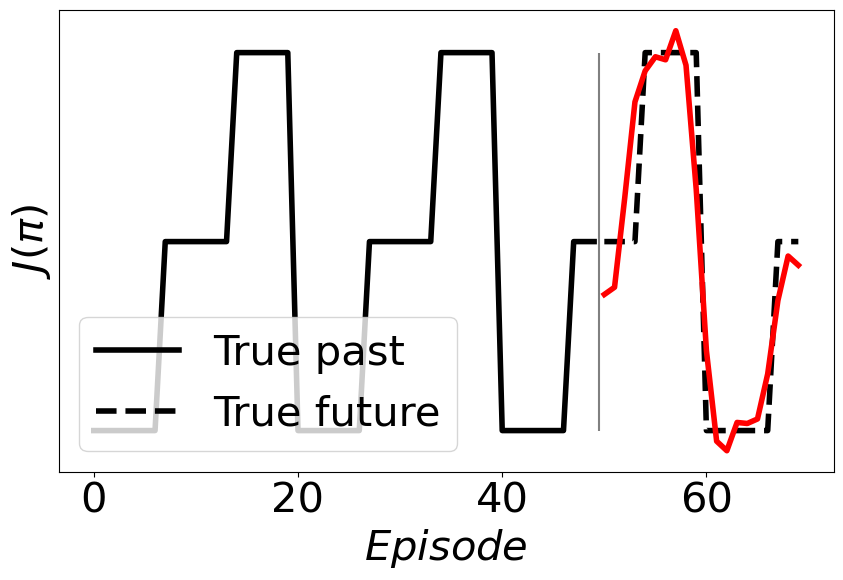}
    \includegraphics[width=0.3\textwidth]{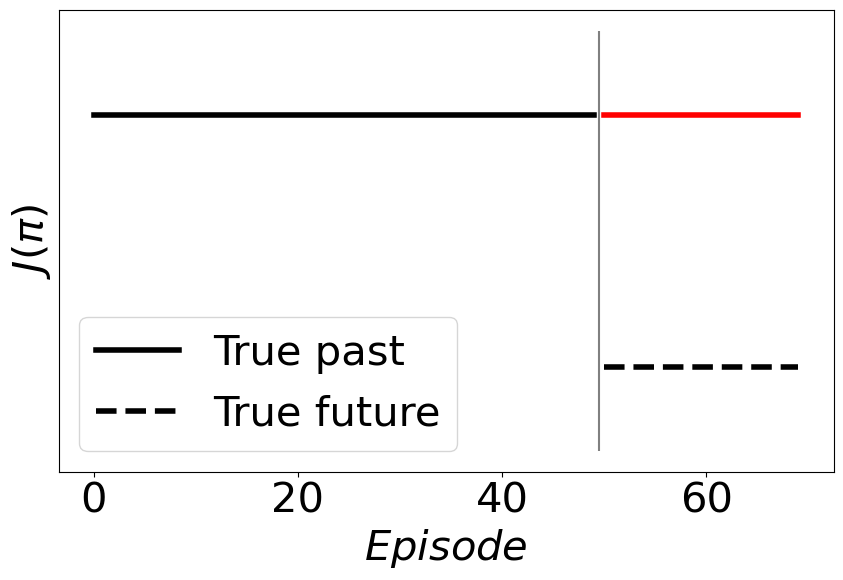}
    \caption{In this figure we plot different kinds of performance trends and discuss the applicability of \thref{ass:fixedf} for each. The red curve corresponds to the forecast obtained using an auto-regressive model. \textbf{(Left)} In many cases where the performance of a policy is smoothly changing over time (for e.g., drifts in interests of an user that a recommender system needs to account for), looking at the past performances can often provide indication of how the performance would evolve in the future.
    \textbf{(Middle)} Changes in performances does not necessarily have to be smooth. What \thref{ass:fixedf} enforces is that the changes have some structure which can be generalized to make predictions about how the performance would change in the future. Here, the performance jumps between different values (for e.g., if there is discontinuous change in the underlying system), but till their is some structure in the changes, it can be leveraged to make predictions about the future performances as well.
    \textbf{(Right)} While \thref{ass:fixedf} can be applicable in many setting, there can be settings where this assumption does not hold. For example, if a motor of an industrial system is degrading over time but this degradation has no effect on the observable performance, until the point when the motor breaks down and the performance drops completely. In such cases, just looking at past performances may not be sufficient to infer how performance will change in the future.
     }
    \label{fig:assex}
\end{figure}

\subsection{Latent Variables}
Instead of enforcing structure on the performances, a possible alternative could have been to enforce structure on how the underlying latent variable (e.g., friction of a motor, interests of a user) are changing over time.
While this might be more intuitive for some, just considering structure on this latent variable need not be sufficient.
Dealing with latent/hidden variables can particularly challenging in the off-policy setting, as it may often not be possible (unless additional assumptions are enforced) to infer the latent variable using just the observations from past interactions, even in the stationary setting \citep{tennenholtz2020off, namkoong2020off,shi2021minimax,bennett2021off}.

Further, the end goal is to estimate the performance of a policy in the future. Therefore, even if we could infer the possible latent variables for the future episodes, it would still require additional regularity conditions on the (unknown) function that maps from the latent variable to the performance associated with it for any given policy.
Without that it would not be possible to generalize what would the performance be for the inferred latent variables of the future.
And as we discuss in Figure \ref{fig:sine}, these two  assumptions on (a) the structure of how the latent variable could change, and (b) the regularity condition on how the latent variable impacts the performance, can often be reduced to a single condition directly on the structure of how the performances are changing.

\begin{figure}[t]
    \centering
    \begin{minipage}{.2\textwidth} 
        \centering
      \includegraphics[width=\textwidth]{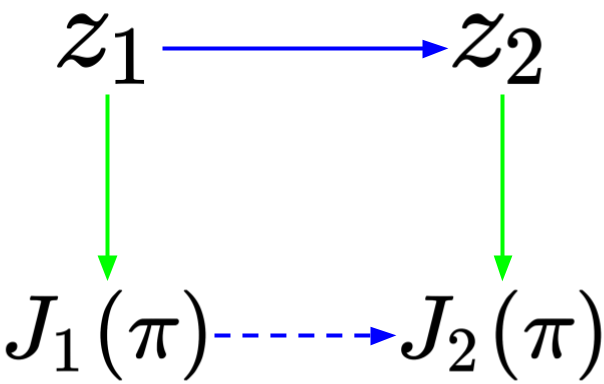}
    \end{minipage}
      \hspace{50pt}
    \begin{minipage}{.3\textwidth}
    \centering
     \includegraphics[width=\textwidth]{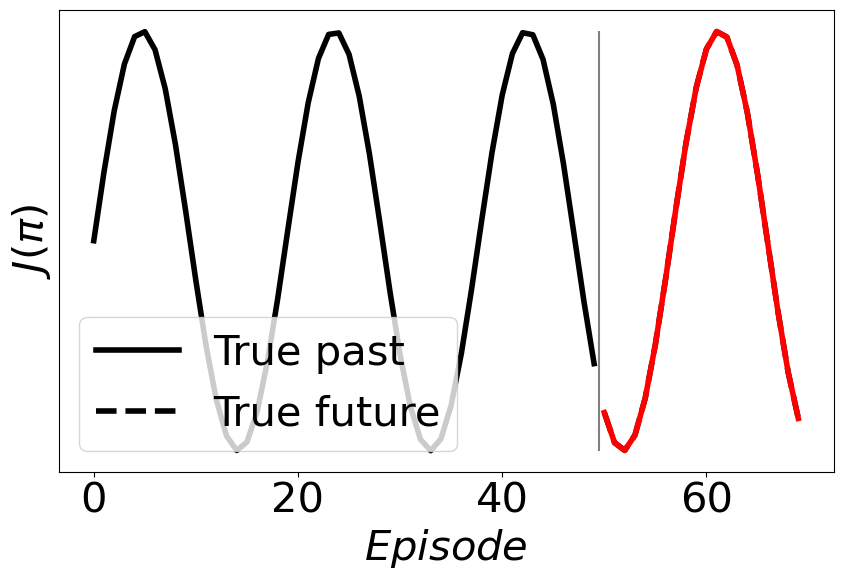}
    \end{minipage}
      %
    \caption{ \textbf{(Left)}
    Considering structured changes in latent variable $z$ (\textcolor{blue}{blue} arrow) of the POMDP might often be more intuitive. 
    However, as $J(\pi)$ estimation is required ultimately, unless performance of a policy also has some structure (\textcolor{green}{green} arrows) given $z$,  generalizing across (potentially unseen) $z$'s may not be possible.
    Structured changes for \textcolor{blue}{blue} and \textcolor{green}{green} arrows consequently results in structured changes in $J(\pi)$ (\textcolor{blue}{dashed-blue} arrows).
    For example, if the performance $J(\pi)$ of a policy changes (Lipschitz) smoothly with $z$, 
    then (Lipschitz) smooth changes between $z$ values automatically also imply (Lipschitz) smooth changes between $J(\pi)$ values.
    \textbf{(Right)} When executing a policy $\pi$, say $z$ changes as $z_i=i$, and $J_i(\pi)$ changes periodically as $\sin(z_i)$.
    Here, even though both $z$ and $J$ change smoothly, changes in $z_{i+1}$ can be modeled using one past term (i.e, $z_{i}$) as given just the current performance value $J_{i}(\pi)$ it is not possible to predict whether the next performance $J_{i+1}(\pi)$ would increase or decrease in the future.
    However, such a problem can be easily resolved by looking at multiple past performances to infer the trend (for e.g., just using past $2$ terms here suffices to exactly predict the future outcomes({\color{red}red} curve)).
    %
    %
    }
    \label{fig:sine}
\end{figure}

\section{Proofs for Theoretical Results}
\label{apx:proofs}

\subsection{Double Counterfactual Reasoning}
\label{apx:proofdis}
\doubleISthm*

\begin{proof}

In the following, to make the dependence of trajectories explicit, we will additionally define $\rho(h)$ and $g(h)$ to be the importance ratios and the return associated with a trajectory $h$. 
Using this notation, it can be observed that,
\begin{align}
     \mathbb{E}_{\pi}\left[J_{i+1}(\pi) |M_i\right] 
          &= \sum_{h_{i+1}} \Pr(h_{i+1}|M_i;\pi) g(h_{i+1})
     \\
     &\overset{(a)}{=} \sum_{h_{i+1}}\sum_{m_{i+1}}\sum_{h_{i}} \Pr(h_{i+1}, m_{i+1}, h_{i}|M_i; \pi) g(h_{i+1})
    \\
     &\overset{(b)}{=} 
        \sum_{h_{i}} \Pr(h_i|M_i;\pi)
        \sum_{m_{i+1}} \Pr(m_{i+1}|h_i,M_i;\pi)
        \\
        &\quad\quad 
        \sum_{h_{i+1}} \Pr(h_{i+1}|m_{i+1}, h_{i}, M_i; \pi) g(h_{i+1})
    \\
     &\overset{(c)}{=} \sum_{h_{i}} \Pr(h_i|M_i;\pi)
        \sum_{m_{i+1}} \Pr(m_{i+1}|h_i,M_i)
        \sum_{h_{i+1}} \Pr(h_{i+1}|m_{i+1}; \pi) g(h_{i+1})
    \\
     &\overset{(d)}{=} \sum_{h_{i}} \rho(h_i) \Pr(h_i|M_i;\beta_k)
        \sum_{m_{i+1}} \Pr(m_{i+1}|h_i,M_i)
        \\
        &\quad\quad 
        \sum_{h_{i+1}} \rho(h_{i+1})\Pr(h_{i+1}|m_{i+1}; \beta_{i+1}) g(h_{i+1})
    \\
     &\overset{(e)}{=} \sum_{h_{i}}\sum_{m_{i+1}}\sum_{h_{i+1}}             
            \Pr(h_i|M_i;\beta_i)
            \Pr(m_{i+1}|h_i,M_i)
            \Pr(h_{i+1}|m_{i+1}; \beta_{i+1})
        \Big[\rho(h_i)\rho(h_{i+1}) g(h_{i+1})\Big]
    \\
     &= \mathbb{E}_{\beta_i \beta_{i+1}}\left[\rho_i \rho_{i+1} G_{i+1} |M_i\right]
     \\
     &= \mathbb{E}_{\beta_i \beta_{i+1}}\left[\rho_i \widehat J_{i+1}(\pi) |M_i\right], \label{eqn:doubleIS}
\end{align}
where (a) follows from the law of total probability, (b) follows from the chain rule of probability, (c) follows using conditional independence, where $m_{i+1}$ is independent of $\pi$ given $h_i$ and $M_i$ because of the meta-transition function $\mathcal T$, and $h_{i+1}$ i independent of $h_i$ and $M_i$ given $m_{i+1}$ and $\pi$, (d) follows from the use of importance sampling to switch the sampling distribution under \thref{ass:support}, and (e) follows from re-arrangement of terms.
Finally, $\rho_i$ and $\rho_{i+1}$ are the random variables corresponding the importance ratios in episodes $i$ and $i+1$.
Random variable $G_{i+1}$ corresponds to the return under $\beta$ in episode $i+1$.

Now notice that
\begin{align}
     \mathbb{E}_{\pi}\left[J_{i+1}(\pi) |M_i\right] &= \sum_{y \in \mathbb R} \Pr(J_{i+1}(\pi)=y|M_i;\pi) y
     \\
     &\overset{(f)}{=}  \sum_{y \in \mathbb R} \Pr(J_{i+1}(\pi)=y|J_i(\pi);\pi) y
     \\
     &= \mathbb{E}_{\pi}\left[J_{i+1}(\pi) |J_i(\pi)\right], \label{eqn:doubleIS2}
\end{align}
where $(f)$ follows from \thref{ass:fixedf}.
Finally, combining \eqref{eqn:doubleIS} and \eqref{eqn:doubleIS2},
\begin{align}
    \mathbb{E}_{\pi}\left[J_{i+1}(\pi) |J_i(\pi)\right]
     = \mathbb{E}_{\beta_i,\beta_{i+1}}\left[\rho_i \widehat J_{i+1}(\pi) \middle| M_{i}\right].
\end{align}
%
\end{proof}

Similarly, under a more generalized \thref{ass:fixedf}, where
    $\forall m \in \mathcal M$, 
    \begin{align}
        \forall i>p, \, \Pr(J_{i+1}(\pi)=j_{i+1}| M_i=m; \pi') = \Pr(J_{i+1}(\pi)=j_{i+1}|(J_{i-k}(\pi)=j_{i-k})_{k=0}^p; \pi').
    \end{align}
then similar steps as earlier can be used to conclude that
\begin{align}
    \mathbb{E}_{\pi}\left[J_{i+1}(\pi) |(J_{i-k}(\pi))_{k=0}^p\right] 
     &= \mathbb{E}_{\beta_i \beta_{i+1}}\brr{\rho_i \widehat J_{i+1}(\pi) |M_i}. 
    \label{eqn:diss2}
\end{align}
Note that no additional importance correction is needed in \eqref{eqn:diss2} compared to \eqref{eqn:doubleIS2}. The term $\rho_i$ only shows up to correct for the transition between $M_i$ and $M_{i+1}$ due to the meta-transition function $\mathcal T(m,h,m')=\Pr(M_{i+1}{=}m'|M_i{=}m,H_i{=}h)$.
This independence on the choice of $p$ also holds if $\mathcal T$ is non-Markovian in the previous $M_i$ values.
Although, additional importance correction would be required if $\mathcal T$ is dependent on multiple past $H_i$ terms.

\subsection{Asymptotic bias of \emph{$\hat \theta_{naive}$}}
\label{proof:bias}

Recall that $\hat \theta_\text{naive}$ is given by,
\begin{align}
        \hat \theta_\texttt{naive} &\in \argmin_{\theta \in \Theta} \,\, \sum_{i=1}^{n-1}\left( f\left(\widehat J_{i}(\pi); \theta\right) -  \rho_{i}\widehat J_{i+1}(\pi) \right)^2. 
\end{align}
Because $\widehat J_i(\pi)$ is an unbiased estimate of $J_{\pi}$, let $\widehat J_i(\pi) = J_i(\pi) + \eta_i$, where $\eta_i$ is a mean zero noise.
Let $\mathbb N \coloneqq [\eta_1, \eta_2, ..., \eta_{n-1}]^\top$ and $
    \mathbb J \coloneqq [J_1(\pi), J_2(\pi), ..., J_{n-1}(\pi)]^\top$.
When $f$ is a linear function of its inputs, expected value $\mathbb E_\pi\brr{J_{i+1}(\pi)|J_{i}(\pi)} = J_i \theta_\pi$.
Also, as $\rho_i \hat J_{i+1}(\pi)$ is an unbiased estimator for $J_i(\pi) \theta_\pi$ given $J_i(\pi)$, let $\rho_i \hat J_{i+1}(\pi)=J_i(\pi) \theta_\pi + \zeta_i$, where $\zeta_i$ is mean zero noise. 
Let $\mathbb N_2 \coloneqq [\zeta_1, \zeta_2, ..., \zeta_{n-1}]^\top$ then
$\theta_\texttt{naive}$ can be expressed as,
\begin{align}
    \hat \theta_\text{naive} &= \br{\br{\mathbb J + \mathbb N}^\top \br{\mathbb J + \mathbb N}}^{-1} \br{\mathbb J + \mathbb N}^\top \br{\mathbb J \theta_\pi + \mathbb N_2}
    \\
    &= \br{\mathbb J^\top\mathbb J + 2\mathbb J^\top\mathbb N + \mathbb N^\top\mathbb N}^{-1} \br{\mathbb J^\top \mathbb J \theta_\pi + \mathbb N^\top\mathbb J \theta_\pi + \mathbb J^\top\mathbb N_2 + \mathbb N^\top\mathbb N_2}
    \\
    &= \br{\frac{1}{n}\br{\mathbb J^\top\mathbb J + 2\mathbb J^\top\mathbb N + \mathbb N^\top\mathbb N}}^{-1} \br{\frac{1}{n}\br{\mathbb J^\top \mathbb J \theta_\pi + \mathbb N^\top\mathbb J \theta_\pi + \mathbb J^\top\mathbb N_2 + \mathbb N^\top\mathbb N_2}} \label{eqn:biaseq}.
\end{align} 
In the limit, using continuous mapping theorem when the inverse in \eqref{eqn:biaseq} exists,
\begin{align}
    \lim_{n\rightarrow\infty}\hat \theta_\text{naive} &= \br{\lim_{n\rightarrow\infty} \frac{1}{n}\br{ \mathbb J^\top\mathbb J + 2\mathbb J^\top\mathbb N + \mathbb N^\top\mathbb N}}^{-1} \br{\lim_{n\rightarrow\infty} \frac{1}{n} \br{\mathbb J^\top \mathbb J \theta_\pi + \mathbb N^\top\mathbb J \theta_\pi + \mathbb J^\top\mathbb N_2 + \mathbb N^\top\mathbb N_2}}. \label{eqn:biaseq2}
\end{align}
Observe that both $\mathbb N$ and $\mathbb N_2$ are mean zero and uncorrelated  with each other and also with $\mathbb J$.
Therefore, the terms corresponding to $\mathbb J^\top \mathbb N$, $\mathbb J^\top \mathbb N_2$, and $\mathbb N^\top \mathbb N_2$ in \eqref{eqn:biaseq2} will be zero almost surely due to Rajchaman's strong law of large numbers for uncorrelated random variables \citep{rajchman1932zaostrzone,chandra1991extensions}.
However, the term corresponding to $\mathbb N^\top \mathbb N$ will not be zero in the limit, and instead roughly result in (average of the) variances of $\eta_i$.
Consequently, this results in,
\begin{align}
    \hat \theta_\texttt{naive} \overset{a.s.}{\longrightarrow} \left(\mathbb J^\top \mathbb J + \mathbb N ^\top \mathbb N \right)^{-1} \mathbb{J}^\top \mathbb J \theta_\pi.
\end{align}

\subsection{Importance-Weighted IV-Regression}
\label{apx:proofiwiv}
\covthm*

\begin{proof}
    \begin{align}
        \forall i, \, \operatorname{Cov}\left(\widehat J_i(\pi), \widehat J_{i+1}(\pi) - J_{i+1}(\pi)\right) &= \underbrace{\mathbb{E}_\beta\left[\widehat J_i(\pi) \left(\widehat J_{i+1}(\pi) - J_{i+1}(\pi)\right) \right]}_{\text{(I)}} 
        \\
        &\quad - \underbrace{\mathbb{E}_\beta\left[\widehat J_i(\pi) \right]\mathbb E_\beta\left[ \widehat J_{i+1}(\pi) - J_{i+1}(\pi) \right]}_{\text{(II)}}. \label{eqn:apx:cov}
    \end{align}
    Focusing on term (II),
    \begin{align}
        \mathbb{E}_\beta\left[\widehat J_i(\pi) \right]\mathbb E_\beta\left[ \widehat J_{i+1}(\pi) - J_{i+1}(\pi) \right] &= \mathbb{E}_\beta\left[\widehat J_i(\pi) \right]\left(\mathbb E_\beta\left[ \widehat J_{i+1}(\pi) \right] - J_{i+1}(\pi) \right)
        \\
         &\overset{(a)}{=}\mathbb{E}_\beta\left[\widehat J_i(\pi) \right]\left( J_{i+1}(\pi) - J_{i+1}(\pi) \right)
         \\
         &=0,
    \end{align}

where (a) follows from the fact that under \thref{ass:support}, $\widehat J_{i+1}(\pi)$ is an unbiased estimator for $J_{i+1}(\pi)$ \citep{thomas2015safe}.
    Focusing on term (I) and using the law of total expectation,
    \begin{align}
        \mathbb{E}_\beta\left[\widehat J_i(\pi) \left(\widehat J_{i+1}(\pi) - J_{i+1}(\pi)\right) \right] &= \mathbb{E}_\beta\Big[\widehat J_i(\pi) \underbrace{\mathbb E_\beta \left[\widehat J_{i+1}(\pi) - J_{i+1}(\pi)\middle| \widehat J_i(\pi) \right]}_{\text{(III)}} \Big].
        %
        %
    \end{align}
    Expanding term (III) further using the law of total expectation,
    \begin{align}
    \mathbb E_\beta \left[\widehat J_{i+1}(\pi) - J_{i+1}(\pi)\middle| \widehat J_i(\pi) \right] &\overset{(b)}{=}     \mathbb E_\beta \left[\mathbb E_{\beta}\left[\widehat J_{i+1}(\pi) - J_{i+1}(\pi)\middle| M_{i+1}, \widehat J_i(\pi) \right] \middle| \widehat J_i(\pi) \right] 
    \\
    &\overset{(c)}{=}     \mathbb E_\beta \left[\mathbb E_{\beta}\left[\widehat J_{i+1}(\pi) - J_{i+1}(\pi)\middle| M_{i+1} \right] \middle| \widehat J_i(\pi) \right] 
    \\
    &\overset{(d)}{=} 0,
    \end{align}
    where in (b) the outer expectation is over the next environment $M_{i+1}$ given that the current performance estimate is $\widehat J_i(\pi)$ and that $\beta_i$ was used for interaction in episode $i$. The inner expectation is over $\widehat J_{i+1}(\pi)$, where the trajectory used for estimating $\widehat J_{i+1}(\pi)$ is collected using $\beta$ in the environment $M_{i+1}$.
    Step (c) follows from the fact that conditioned on the environment $M_{i+1}$, interactions in $M_{i+1}$ are independent of quantities observed in the episodes before $i+1$.
    Finally, step (d) follows from observing that
    \begin{align}
        \mathbb E_{\beta}\left[\widehat J_{i+1}(\pi) - J_{i+1}(\pi)\middle| M_{i+1} \right] &= \mathbb E_{\beta}\left[\widehat J_{i+1}(\pi)\middle| M_{i+1} \right] -  J_{i+1}(\pi)
        \\
        &\overset{(e)}{=} J_{i+1}(\pi) -  J_{i+1}(\pi)
        \\
        &= 0,
    \end{align}
    where (e) follows from the fact that under \thref{ass:support}, $\widehat J_{i+1}(\pi)$ is an unbiased estimator of the performance of $\pi$ for the given environment $M_{i+1}$.
    Therefore both (a) and (b) in \eqref{eqn:apx:cov} are zero, and we conclude the result.
\end{proof}

\consistencythm*

\begin{proof}
For the linear setting, $\hat \theta_n$ can be expressed as,
\begin{align}
    \hat \phi_n &\in \argmin_{\phi \in \Phi} \,\, \sum_{i=2}^{n-1}\left( \widehat J_{i-1}(\pi) \phi - \widehat J_{i}(\pi) \right)^2. \label{apx:s1}
    \\
    \hat \theta_n &\in \argmin_{\theta \in \Theta} \,\, \sum_{i=2}^{n-1}\left( \bar J_{i}(\pi) \theta - \rho_{i}\widehat J_{i+1}(\pi) \right)^2, & \text{where} \quad \bar J_{i} \coloneqq \widehat J_{i-1}\hat \phi_n. \label{apx:s2}
\end{align}

%

Before moving further, we introduce some additional notations.
Particularly, we will use matrix based notations such that it provides more insights into how the steps would work out for other choices of instrument variables as well. 
\begin{align}
    \bf X_{1} &\coloneqq \left[\widehat J_1(\pi), ..., \widehat J_{n-2}(\pi)\right]^\top, 
    %
    & \bf \Lambda_{1} &\coloneqq \texttt{diag}([\rho_1, ..., \rho_{n-2}]), 
    \\
    \bf X_{2} &\coloneqq \left[ \widehat J_{2}(\pi), ...,  \widehat J_{n-1}(\pi)\right]^\top,
    &\bf \Lambda_{2} &\coloneqq \texttt{diag}\left(\left[\rho_2, ..., \rho_{n-1}\right]\right), 
    \\
    \bf X_{3} &\coloneqq \left[ \widehat J_{3}(\pi), ...,  \widehat J_{n}(\pi)\right]^\top 
    & \bf \bar X_{2} &\coloneqq \left[\bar J_2(\pi), ..., \bar J_{n-1}(\pi)\right]^\top,
\end{align}
where the $\texttt{diag}$ corresponds to a diagonal matrix with off-diagonals set to zero.

In the following, we split the proof in two parts: (a) we will first show that 
    \begin{align}
        \hat \theta_n 
        &= \left(\bf X_{1}^\top \bf X_{2} \right)^{-1} \left(\bf X_{1} ^\top  \bf \Lambda_{2}  \bf X_{3}\right), 
    \end{align}
and then (b) using this simplified form for $\hat \theta_n$ we will show that $\hat \theta_n   \overset{\text{a.s.}}{\longrightarrow} \theta_\pi$.

\paragraph{Part (a)} Solving \eqref{apx:s1} in matrix form,
    \begin{align}
        \hat \phi_n &=  \left(\bf X_{1}^\top \bf X_{1} \right)^{-1} \bf X_{1}^\top \bf \bf X_{2}. \label{eqn:phihat}
    \end{align}
Similarly, solving \eqref{apx:s2} in matrix form,
    \begin{align}
        \hat \theta_n &= \left(\bf \bar  X_{2}^\top \bf \bar X_{2} \right)^{-1} \bf \bar X_{2}^\top \bf  \Lambda_{2}  \bf X_{3}. \label{eqn:hattheta}
    \end{align}
    Now substituting the value of $\bf \bar X_{2}$ in \eqref{eqn:hattheta},
    \begin{align}
        \hat \theta_n &= \left(\left(\underbrace{\bf X_{1} \hat \phi_n}_{\bf \bar X_{2}}\right)^\top \left(\underbrace{\bf X_{1} \hat \phi_n}_{\bf \bar X_{2}}\right) \right)^{-1} \left(\underbrace{\bf X_{1} \hat \phi_n}_{\bf \bar X_{2}}\right)^\top \bf  \Lambda_{2}  \bf X_{3}. \label{eqn:sub1}
    \end{align}
     Using \eqref{eqn:phihat} to substitute the value of $\hat \phi_n$ in \eqref{eqn:sub1},
    \begin{align}
        \hat \theta_n &= \left( \left(\bf X_{1} \underbrace{\left(\bf X_{1}^\top \bf X_{1} \right)^{-1} \bf X_{1}^\top \bf X_{2}}_{\hat \phi_n} \right)^\top \left(\bf X_{1} \underbrace{\left(\bf X_{1}^\top \bf X_{1} \right)^{-1} \bf X_{1}^\top  \bf X_{2}}_{\hat \phi_n}\right) \right)^{-1}
        \\
        & \quad \left(\bf X_{1} \underbrace{\left(\bf X_{1}^\top \bf X_{1} \right)^{-1} \bf X_{1}^\top \bf X_{2}}_{\hat \phi_n} \right)^\top \bf  \Lambda_{2}  \bf X_{3}. \label{eqn:trans}
    \end{align}
    Using matrix operations to expand the transposes in \eqref{eqn:trans},
    \begin{align}
        \hat \theta_n &= \left(\left(\mul{red}{\bf X_{2}^\top  \bf X_{1}} \left(\mul{blue}{\bf X_{1}^\top \bf X_{1}} \right)^{-1} \mul{green}{\bf X_{1}^\top} \right)\left(\mul{green}{\bf X_{1}} \left(\mul{purple}{\bf X_{1}^\top \bf X_{1}} \right)^{-1} \mul{brown}{\bf X_{1}^\top  \bf X_{2}} \right)  \right)^{-1}
        \\
        & \quad \left(\bf X_{2}^\top \bf X_{1} \left(\bf X_{1}^\top \bf X_{1} \right)^{-1} \bf X_{1} ^\top \right) \bf  \Lambda_{2}  \bf X_{3}. \label{eqn:tinv}
    \end{align}
    Similarly, using matrix operations to expand inverses in \eqref{eqn:tinv} (colored underlines are used to match the terms before expansion in \eqref{eqn:tinv} and after expansion in \eqref{eqn:cut}),
    \begin{align}
        \hat \theta_n &= \left(\mul{brown}{\bf X_{1}^\top  \bf X_{2}} \right)^{-1}  \left(\mul{purple}{\bf X_{1}^\top \bf X_{1}} \right) \left(\mul{green}{\bf X_{1}^\top \bf X_{1}} \right)^{-1} \left(\mul{blue}{\bf X_{1}^\top \bf X_{1} }\right)  \left(\mul{red}{\bf X_{2}^\top  \bf X_{1}} \right)^{-1}
        \\
        & \quad \left(\bf X_{2}^\top  \bf X_{1} \right) \left(\bf X_{1}^\top \bf X_{1} \right)^{-1} \left(\bf X_{1} ^\top  \bf  \Lambda_{2}  \bf X_{3}\right),
        \label{eqn:cut}
    \end{align}
    Notice that several terms in \eqref{eqn:cut} cancel each other out, therefore,
    \begin{align}
        \hat \theta_n 
        &= \left(\bf X_{1}^\top \bf X_{2} \right)^{-1} \left(\bf X_{1} ^\top  \bf  \Lambda_{2}  \bf X_{3}\right).  \label{eqn:simpleIV}
    \end{align}

As a side remark, we note that if we replace $\bf X_{1}$ in the above steps with an appropriate instrument variable $\bf Z_{1}$, then similar steps will follow
and will result in
    \begin{align}
        \hat \theta_n 
        &= \left(\bf Z_{1}^\top \bf X_{2} \right)^{-1} \left(\bf Z_{1} ^\top  \bf  \Lambda_{2}  \bf X_{3}\right). \label{eqn:simpleIV2}
    \end{align}

\paragraph{Part (b)} Now when $f(J_i(\pi); \theta_\pi) \coloneqq \mathbb{E}_{\pi}\left[J_{i+1}(\pi) |J_i(\pi)\right]$ is a linear function, 
\begin{align}
    J_{i+1}(\pi) &=   J_i(\pi)\theta_\pi +   U_{i+1}(H_i),
\end{align}
where $U_{i+1}$ is a bounded mean zero noise (which depends on the interaction $H_i$ by $\pi$).
Using \thref{lemma:doubleIS}, let $$Y_{i+1} \coloneqq \mathbb{E}_\pi\left[J_{i+1}(\pi)\middle| J_i(\pi)\right]$$ and its unbiased estimate be
\begin{align}
 \widehat Y_{i+1} \coloneqq \rho_{i} \widehat J_{i+1}(\pi) = \rho_i \rho_{i+1} G_{i+1}. \label{eqn:SEMs}
\end{align}
For the regression, since $\widehat J_i(\pi)$ is an unbiased estimate of the input $J_i(\pi)$ and $\widehat Y_{i+1}$ is an unbiased estimate of the target $\mathbb{E}_\pi\left[J_{i+1}(\pi)\middle| J_i(\pi)\right]$, these can be equivalently expressed as, 
\begin{align}
    \widehat J_{i}(\pi) &=   J_{i}(\pi) +   V_{i}(H_{i}),
    \\
    \widehat Y_{i+1} &= J_{i+1}(\pi) + W_{i+1}(H_i, H_{i+1}), \label{eqn:temppp}
\end{align}
where $V_{i}(H_{i})$ is some bounded mean-zero noise (dependent on the unbiased estimate made using $H_i$) and $ W_{i+1}(H_i, H_{i+1})$ is also a bounded mean-zero noise (dependent on the unbiased estimate made using $H_i$ and $H_{i+1}$).
Before moving further, we define some additional notation,
\begin{align}
    \bf Y_{3} &\coloneqq [Y_3, ..., Y_{n}]^\top &\bf U_{3} &\coloneqq [U_3(H_2), ..., U_{n}(H_{n-1})]^\top, 
    \\
    \bf \widehat Y_{3} &\coloneqq [\widehat Y_3, ..., \widehat Y_{n}]^\top & \bf V_{2} &\coloneqq [V_2(H_2), ..., V_{n-1}(H_{n-1})]^\top. 
    \\
    \bf \mathbb J_{2} &\coloneqq \left[J_2(\pi), ..., J_{n-1}(\pi) \right]^\top & \bf W_{3} &\coloneqq [W_3(H_2,H_3), ..., W_{n}(H_{n-1}, H_n)]^\top. 
\end{align}

Using \eqref{eqn:SEMs} note that $\bf \widehat Y_{3} = \bf \Lambda_{2} \bf X_{3}$, therefore \eqref{eqn:simpleIV} can be expressed as,
\begin{align}
    \hat \theta_n &= \left(\bf X_{1}^\top \bf X_{2} \right)^{-1} \left(\bf X_{1} ^\top   \bf \widehat Y_{3}\right). \label{eqn:s1}
    \end{align}
    Unrolling value of $\bf \widehat Y_{3}$ in \eqref{eqn:s1} using relations from \eqref{eqn:SEMs} and \eqref{eqn:temppp},
    \begin{align}
    \hat \theta_n &= \left(\bf X_{1}^\top  \bf X_{2} \right)^{-1} \left(\bf X_{1} ^\top   \left(\bf Y_{3} + \bf    W_{3} \right)\right)
        \\
        &= \left(\bf X_{1}^\top  \bf X_{2} \right)^{-1} \left(\bf X_{1} ^\top   \left(\mathbb J_{2}\theta_\pi +  \bf  U_{3} + \bf   W_{3} \right)\right)
        \\
        &= \left(\bf X_{1}^\top  \bf X_{2} \right)^{-1} \left(\bf X_{1} ^\top   \left( \left(\bf X_{2} - \bf   V_{2} \right)\theta_\pi +  \bf  U_{3} + \bf   W_{3} \right)\right). \label{eqn:s2}
    \end{align}
    Expanding \eqref{eqn:s2},
    \begin{align}
        \hat \theta_n &= \theta_\pi  - \left(\bf X_{1}^\top  \bf X_{2} \right)^{-1}\bf X_{1} ^\top \bf   V_{2}\theta_\pi   + \left(\bf X_{1}^\top  \bf X_{2} \right)^{-1} \left(\bf X_{1} ^\top   \left(\bf   U_{3} + \bf   W_{3} \right)\right). \label{eqn:exp}
\end{align}
Evaluating the value of \eqref{eqn:exp} in the limit,
\begin{align}
    \lim_{n \rightarrow \infty} \hat \theta_n =  \theta_\pi - \lim_{n \rightarrow \infty} \left(   \underbrace{\left(\bf X_{1}^\top  \bf X_{2} \right)^{-1}\bf X_{1} ^\top \bf   V_{2}\theta_\pi }_{(a)} + \underbrace{\left(\bf X_{1}^\top  \bf X_{2} \right)^{-1} \left(\bf X_{1} ^\top   \left(\bf   U_{3} + \bf   W_{3} \right)\right)}_{(b)} \right). \label{eqn:lim}
\end{align}
It can be now seen from \eqref{eqn:lim} that if in the limit the terms inside the paranthesis are zero, then we would obtain our desired result. 
Focusing on the term (a) and using the continuous mapping theorem,
\begin{align}
     \lim_{n \rightarrow \infty} \left(\bf X_{1}^\top \bf X_{2} \right)^{-1}\bf X_{1} ^\top \bf   V_{2}\theta_\pi &=  \lim_{n \rightarrow \infty} \left(\frac{1}{n}\bf X_{1}^\top  \bf X_{2} \right)^{-1} \left( \frac{1}{n} \bf X_{1} ^\top \bf   V_{2}\theta_\pi \right)
     \\
     &=   \left(\lim_{n \rightarrow \infty} \frac{1}{n}\bf X_{1}^\top  \bf X_{2} \right)^{-1} \left(\underbrace{\lim_{n \rightarrow \infty} \frac{1}{n} \bf X_{1} ^\top \bf   V_{2}}_{(c)}\right) \theta_\pi , \label{eqn:zero}
\end{align}
%
where \thref{ass:correlated} ensures that $\bf X_1$ and $\bf X_2$ are correlated and thus their dot product is not zero.
Notice that term (c) \eqref{eqn:zero} can be expressed as $\frac{1}{n}\sum_{i=2}^{n-1} X_{i-1}   V_{i}$.
Further, recall from \thref{thm:cov} that $  V_{i}$ is a mean zero random variable uncorrelated with $X_{i-1}$ for all $i$. Further, $V_{i}$ and $X_{i-1}$ are also bounded for all $i$ as both rewards and importance ratios are bounded (\thref{ass:support}), and $T$ is finite.
Now, for $\alpha_i \coloneqq X_{i-1}V_i$ observe that $\mathbb E\brr{\alpha_i} = \mathbb{E} \brr{ X_{i-1}\mathbb E\brr{V_i|X_{i-1}}} = \mathbb E\brr{X_{i-1}0} = 0$ and thus $\alpha_i$ is a bounded and mean zero random variable $\forall i$.
%
Therefore, as $(c)$ is an average of $\alpha$ variables, it follows from the Rajchaman's strong law of large numbers for uncorrelated random variables \citep{rajchman1932zaostrzone,chandra1991extensions} that term under $(c)$ is zero almost surely.
Thus,
\begin{align}
    \left(\bf X_{1}^\top \bf X_{2} \right)^{-1}\bf X_{1} ^\top \bf   V_{2}\theta_\pi \overset{\text{a.s.}}{\longrightarrow} 0.
\end{align}
%

Similarly, for term (b) in \eqref{eqn:lim} observe that both $\bf  U_3$ and $\bf  W_{3}$ are zero mean random variables uncorrelated with $\bf X_{1}$.
Therefore, term (b) in \eqref{eqn:lim} is also zero in the limit almost surely.
It can now be concluded from \eqref{eqn:lim} that
\begin{align}
    \hat \theta_n   \overset{\text{a.s.}}{\longrightarrow} \theta_\pi.
\end{align}
%





\end{proof}

\wconsistencythm*

\begin{proof}
For the linear setting, $\tilde \theta_n$ can be expressed as,
\begin{align}
    \hat \phi_n &\in \argmin_{\phi \in \Phi} \,\, \sum_{i=2}^{n-1}\rho_i \left( \widehat J_{i-1}(\pi) \phi -  G_{i}(\pi) \right)^2. \label{apx:ws1}
    \\
    \tilde \theta_n &\in \argmin_{\theta \in \Theta} \,\, \sum_{i=2}^{n-1}\rho_{i}\rho_{i+1}\left( \bar J_{i}(\pi) \theta -  G_{i+1}(\pi) \right)^2, & \text{where} \quad \bar J_{i} \coloneqq \widehat J_{i-1}\hat \phi_n. \label{apx:ws2}
\end{align}
Notice that as dividing the objective by a positive constant does not change the optima, we drop the denominator terms in $$\bar \rho_i \coloneqq \frac{\rho_{i}\rho_{i+1}}{\sum_{j=2}^{n-1}\rho_j \rho_{j+1}}$$ for the purpose of the analysis.
Before moving further, we introduce some additional notations besides the ones introduced in the proof of \thref{thm:consistent},
\begin{align}
    \bf G_{3} &\coloneqq \left[ G_{3}, ...,   G_{n}\right]^\top 
    & \bf \bar \Lambda_{2} &\coloneqq \texttt{diag}\br{\left[\rho_2\rho_3, \rho_3\rho_4 ..., \rho_{n-1}\rho_{n}\right]},
\end{align}

Solving \eqref{apx:ws1} in matrix form,
    \begin{align}
        \hat \phi_n &=  \left(\bf X_{1}^\top \bf \Lambda_{2} \bf X_{1} \right)^{-1} \bf X_{1}^\top \Lambda_{2} \bf G_{2}. 
        \\
        &=  \left(\bf X_{1}^\top \bf \Lambda_{2} \bf X_{1} \right)^{-1} \bf X_{1}^\top  \bf X_{2}. \label{eqn:wphihat}
    \end{align}
Similarly, solving \eqref{apx:ws2} in matrix form,
    \begin{align}
        \tilde \theta_n &= \left(\bf \bar  X_{2}^\top  \bf \bar \Lambda_{2} \bf \bar X_{2} \right)^{-1} \bf \bar X_{2}^\top  \bf \bar \Lambda_{2} \bf G_{3}.
        \\
        &\overset{(a)}{=} \left(\bf \bar  X_{2}^\top  \bf \bar \Lambda_{2} \bf \bar X_{2} \right)^{-1} \bf \bar X_{2}^\top  \bf \Lambda_{2} \bf X_{3},
        \label{eqn:whattheta}
    \end{align}
    where (a) follows from the fact that $\rho_i\rho_{i+1}G_{i+1} = \rho_i \widehat J_{i+1}(\pi)$.
    Now substituting the value of $\bf \bar X_{2}$ in \eqref{eqn:whattheta} similar to \eqref{eqn:sub1} and \eqref{eqn:trans} in the proof of \thref{thm:consistent},
    \begin{align}
        \tilde \theta_n &= \left(\left(\mul{red}{\bf X_{2}^\top  \bf X_{1}} \left(\mul{blue}{\bf X_{1}^\top \bf \Lambda_{2} \bf X_{1}} \right)^{-1} \mul{green}{\bf X_{1}^\top} \right) \mul{green}{\bf \bar \Lambda_{2}} \left(\mul{green}{\bf X_{1}} \left(\mul{purple}{\bf X_{1}^\top \bf \Lambda_{2} \bf X_{1}} \right)^{-1} \mul{brown}{\bf X_{1}^\top  \bf X_{2}} \right)  \right)^{-1}
        \\
        & \quad \left(\bf X_{2}^\top \bf X_{1} \left(\bf X_{1}^\top \bf X_{1} \right)^{-1} \bf X_{1} ^\top \right) \bf  \Lambda_{2}  \bf X_{3}. \label{eqn:wtinv}
    \end{align}
    Similarly, using matrix operations to expand inverses in \eqref{eqn:wtinv} (colored underlines are used to match the terms before expansion in \eqref{eqn:wtinv} and after expansion in \eqref{eqn:wcut}) and multiplying and dividing by $n$,
    \begin{align}
        \tilde \theta_n &= \left(\mul{brown}{\bf X_{1}^\top  \bf X_{2}} \right)^{-1}  \left(\mul{purple}{\frac{1}{n}\bf X_{1}^\top \bf \Lambda_{2} \bf X_{1}} \right) \left(\frac{1}{n}\mul{green}{\bf X_{1}^\top \bf \bar \Lambda_{2} \bf X_{1}} \right)^{-1} \left(\frac{1}{n}\mul{blue}{\bf X_{1}^\top \bf \Lambda_{2}\bf X_{1} }\right)  
        \\
        & \left(\mul{red}{\bf X_{2}^\top  \bf X_{1}} \right)^{-1}\quad \left(\bf X_{2}^\top  \bf X_{1} \right) \left(\frac{1}{n}\bf X_{1}^\top \bf \Lambda_{2} \bf X_{1} \right)^{-1} \left(\bf X_{1} ^\top  \bf  \Lambda_{2}  \bf X_{3}\right).
        \label{eqn:wcut}
    \end{align}
    Now focusing on the term underlined in green, in the limit,
    %
    \begin{align}
        \lim_{n\rightarrow\infty} \frac{1}{n}\bf X_{1}^\top \bf \bar \Lambda_{2} \bf X_{1} &= \lim_{n\rightarrow\infty} \frac{1}{n}\sum_{i=2}^{n-1} \rho_{i}\rho_{i+1} \widehat J_{i-1}(\pi) \widehat J_{i-1}(\pi) ^\top
        \\
        &\overset{(a)}{=} \lim_{n\rightarrow\infty} \frac{1}{n}\sum_{i=2}^{n-1} \mathbb E_{\beta_i,\beta_{i+1}}\brr{\rho_{i}\rho_{i+1}} \widehat J_{i-1}(\pi) \widehat J_{i-1}(\pi)^\top  + \frac{1}{n}\sum_{i=2}^{n-1}  \varepsilon_{i} \widehat J_{i-1}(\pi) \widehat J_{i-1}(\pi)^\top  
        \\
        &\overset{(b)}{=} \lim_{n\rightarrow\infty} \frac{1}{n}\sum_{i=2}^{n-1}  \widehat J_{i-1}(\pi) \widehat J_{i-1}(\pi) ^\top
        \\
        &= \lim_{n\rightarrow\infty} \frac{1}{n} \bf X_{1}^\top \bf X_{1}, \label{eqn:wcut2}
    \end{align}
    where in (a) we defined random variable $\rho_i\rho_{i+1}$ as its expected value $E_{\beta_i,\beta_{i+1}}\brr{\rho_{i}\rho_{i+1}}$ plus a mean zero noise $\varepsilon_i$. Step (b) follows from first observing that $\rho_i$ and $\rho_{i+1}$ are uncorrelated. Therefore $\mathbb E_{\beta_i,\beta_{i+1}} \brr{\rho_{i}\rho_{i+1}} = \mathbb E_{\beta_i}\brr{\rho_{i}}\mathbb E_{\beta_{i+1}}\brr{\rho_{i+1}} = 1$ as the expected value of importance ratios is $1$ \citep{thomas2015safe}.
    Similarly, $\varepsilon_i$ is uncorrelated with $\widehat J_{i-1}(\pi)$, i.e., the expected value $\mathbb{E}_{\beta_i,\beta_{i+1}}\brr{\varepsilon_i|\widehat J_{i-1}(\pi)}=\mathbb{E}_{\beta_i,\beta_{i+1}}\brr{\varepsilon_i} = 0$ for any given $J_{i-1}(\pi)$.
    (Intuitively, this step can be seen analogous to the derivation of PDIS, where the expected value of future IS ratios is always one, irrespective of the past events that it has been conditioned on).
    Now notice that the random variable $\zeta_i \coloneqq  \varepsilon_{i} \widehat J_{i-1}(\pi) \widehat J_{i-1}(\pi)^\top $ is bounded and has mean zero for all $i$.
    Therefore, while $\zeta_i$ and $\zeta_j$ may be dependent, they are uncorrelated for all $i\neq j$.
    Using strong law of large number for uncorrelated random variables \citep{rajchman1932zaostrzone,chandra1991extensions} the second term in (a) is zero almost surely.

    Similarly, it can be observed that  $\frac{1}{n}\bf X_{1}^\top \bf \Lambda_{2}  \bf X_{1}$ converges to $\frac{1}{n}\bf X_{1}^\top  \bf X_{1}$.
    Therefore using \eqref{eqn:wcut2} in \eqref{eqn:wcut}, and using the continuous mapping theorem,
    \begin{align}
        \tilde \theta_n &\overset{a.s.}{\longrightarrow} \left(\mul{brown}{\bf X_{1}^\top  \bf X_{2}} \right)^{-1}  \left(\mul{purple}{\frac{1}{n}\bf X_{1}^\top  \bf X_{1}} \right) \left(\frac{1}{n}\mul{green}{\bf X_{1}^\top\bf X_{1}} \right)^{-1} \left(\frac{1}{n}\mul{blue}{\bf X_{1}^\top \bf X_{1} }\right)  \left(\mul{red}{\bf X_{2}^\top  \bf X_{1}} \right)^{-1}
        \\
        & \quad \left(\bf X_{2}^\top  \bf X_{1} \right) \left(\frac{1}{n} \bf X_{1}^\top \bf X_{1} \right)^{-1} \left(\bf X_{1} ^\top  \bf  \Lambda_{2}  \bf X_{3}\right).
        \label{eqn:wwcut}
    \end{align}
    Notice that \thref{ass:correlated} ensures that $\bf X_1$ and $\bf X_2$ are correlated and thus their dot product is not zero. Further, several terms in \eqref{eqn:wwcut} cancel each other out, therefore,
    \begin{align}
        \tilde \theta_n 
        &\overset{a.s.}{\longrightarrow} \left(\bf X_{1}^\top \bf X_{2} \right)^{-1} \left(\bf X_{1} ^\top  \bf  \Lambda_{2}  \bf X_{3}\right).  \label{eqn:wsimpleIV}
    \end{align}
    Now proof can be completed similarly to the part (b) of the proof of \thref{thm:consistent}.
\end{proof}




\section{Empirical Details}

\label{apx:arp}

The code for all the algorithms and experiments can be found here \href{https://github.com/yashchandak/activeNS}{https://github.com/yashchandak/activeNS}


\subsection{Algorithm}

\label{apx:algo}

In Section \ref{sec:MFPE} we  established the key insight for how to forecast the next performance based on a single previous performance, when the true performance trend of a policy can be modeled auto-regressively using a single past term.
However, as noted in \thref{rem:p} and Figure \ref{fig:sine} using more terms can provide more flexibility in the the type of trends that can be modeled. 
Therefore, we leverage statistics based on multiple past terms to form the instrument variable $Z_i$.

One immediate choice for $Z_i$ is $\widehat J_i(\pi)$.
However, we found that the high variance of IS estimate makes $\widehat J_i(\pi)$ a weak instrument variable \citep{pearl2000models}, that is not strongly correlated with $J_{i+1}(\pi)$.
Better choices of $Z_i$ may be the ones that are strongly correlated with $J_{i+1}(\pi)$ but uncorrelated with the noise in the $\widehat J_{i+1}(\pi)$ estimate.
We found that an alternate choice of $Z_i$ composed of the unweighted return $G_i$ and a WIS-like estimate for $J_i(\pi)$ (where the normalization is done only using the importance ratios from episodes before $i$) to be more useful. Specifically, we let $Z_i \coloneqq [G_i, \widetilde J_i(\pi)]$, where
\begin{align}
    \widetilde J_i(\pi) \coloneqq \frac{\rho_i G_i}{\frac{n}{i}\sum_{k=1}^{i} \rho_k}.
\end{align}

It can be observed similar to \thref{thm:cov} that this $Z_i$ is uncorrelated with the noise in $\widehat J_{i+1}(\pi)$ as well.
Further, the weighted version $\widetilde J_i(\pi)$ suffers less from variance and we found it to be more strongly correlated with $J_{i+1}(\pi)$.
Further, often the performance of the behavior policy is positively/negatively correlated with the performance of the evaluation policy and thus $G_i$ tends to be correlated with $J_{i+1}(\pi)$ as well. 
%
%
%
One could also explore other potential IVs; we leave this for future work.

Now using past $p$ values of $Z_i$ to form the complete instrument variable, where $p$ is a hyper-parameter, we use the following importance weighted instrument-variable regression,
\begin{align}
    \tilde \varphi_n &\in \argmin_{\varphi \in \Omega} \,\, \sum_{i={p+1}}^{n} \bar \rho_i \left( g\left(\left(Z_{j}(\pi)\right)_{j=i-p}^{i-1}; \varphi\right) -  G_{i}(\pi) \right)^2
    \label{eqn:apx:gaa}, 
    \\
    \tilde \theta_n &\in \argmin_{\theta \in \Theta}  \sum_{i=2p}^{n-1} \rho_i^\dagger \left( f\left( \left(\bar J_j(\pi) \right)_{j=i-p+1}^{i}; \theta\right) -  G_{i+1}(\pi) \right)^2,\label{eqn:apx:s2aa}
    \end{align}
where,
    \begin{align}
    \bar J_{i}(\pi) = &g\left( \left(Z_{j}(\pi)\right)_{j=i-p}^{i-1}; \tilde \varphi_n\right), \quad \forall p<i\leq n,
    \\
    \bar \rho_i \coloneqq& \frac{\rho_{i}}{(\sum_{j=2}^{n}\rho_j)}
    \quad 
    \\
    \rho_i^\dagger \coloneqq& \frac{\rho_{i}\rho_{i+1}}{(\sum_{j=2}^{n-1}\rho_j \rho_{j+1})}.
\end{align}
Once $\tilde \theta_n$ is obtained, we use it to auto-regressively forecast the future performances.
Particularly, we use $(\bar J_k)_{k=n+1}^{n+L}$ as the predicted performances for the next $L$ episodes, where
\begin{align}
\forall i>n, \,    \bar J_i \coloneqq f\left( \left(\bar J_{i-k}(\pi) \right)_{k=1}^{p}; \tilde \theta_n\right).
\end{align}

While our theoretical results were established for the setting where there is only a single regressor ($p=1$), a more generalized theoretical result for $p>1$ may be possible using the concepts of endogenous and exogenous regressors. Particularly, let $[..., X_{i}, X_{i+1}, X_{i+2}, X_{i+3}, ...]$, be observations from an $AR(2)$ time-series sequence where $X_{i+3}$ depends on $X_{i+1}$ and $X_{i+2}$. Here, using $X_{i+1}$ as the only instrument variable for $X_{i+2}$ is not possible as $X_{i+3}$ is correlated with $X_{i+1}$. However, $Z = X_i$ or even $Z= [X_{i}, X_{i+1}]$ may form a valid instrument for $X_{i+2}$ as neither the noise in $X_{i+3}$ nor the noise in $X_{i+2}$ is correlated with at least one component of $Z$, i.e., $X_{i}$.
For precise instrument relevance conditions and additional discussion, we refer the reader to the works by \citet{abbottIV,cameronIV,parkerIV}. We leave this theoretical extension for the future work.


\subsection{Implementation and Hyper-parameters}


For the Pro-WLS baseline, we use the weighted least-squares procedure using the Fourier basis features \citep{chandak2020optimizing}.
The hyper-parameter for this baseline is the number of Fourier terms $d$ that should be used to estimate the performance trend.
We found that setting $d$ to be too high results in extremely high-variance and setting it to a lower value fails to capture the trend in performance.
Therefore, based on ablation studies in Figure \ref{fig:ablation} we set $d=5$ for all the experiments.

WIS estimator uses all the data form the past. In comparison, for sliding windows WIS (SWIS), we set the sliding window length to be $400$ ($20\%$ of the number of episodes in the data) for all the experiments. That is, SWIS use past $400$ episodes to estimate the future performance.

For OPEN, the hyper-parameter corresponds to the number of terms to condition on during auto-regression.
Similar to SWIS, we set $p=400$ ($20\%$ of the number of episodes in the data) for all the experiments. That is the AR estimator uses past $400$ episodes to predict the performance in the next episode.
For the two stage regression, we observed that choice of learning rate, and avoiding over-fitting (using early-stopping) to be important as well.

For each environment, we collect data consisting of $2000$ episodes of interaction using the behavior policy, and predict the expected future returns if executing the evaluation policy for the next $200$ episodes.
The behavior policy and the evaluation policy for each domain are described in Section \ref{sec:envdet}.

Since the future outcomes are stochastic, to evaluate the true expected future performance in \eqref{eqn:obj}, we create digital-clones of the environment \textit{after} data has been collected using the behavior policy.
Using these clones, we compute the average of $30$ possible futures when executing the evaluation policy.
This estimate of the \textit{expected} future returns are then used as the ground truth for comparison with the predictions made by the algorithms.

For Figure \ref{fig:active_plots_bias}, $|\text{bias}|$ was computed using the absolute value of the difference between  (a) the predicted future performance averaged across 30 trials and (b) the ground truth future performance. That is, for an estimator $\hat J$ of $J$, the bias is  $|J - E[\hat J]|$. Because of this, $30$ trials only gives us a point estimate for bias. 
(Notice that using the absolute value of the difference between  (a) the predicted future performance for each trial and (b) the true future performance’, averaged across $30$ trials, will provide an estimate of $E[|J - \hat J|]$, which would not capture the bias but will be more like the variance (using L1/absolute distance instead of L2)).

\subsection{Environments}
\label{sec:envdet}
We provide empirical results on four non-stationary environments: a toy robot environment, non-stationary mountain car, diabetes
treatment, and  MEDEVAC domain for routing air ambulances.
Details for each of these
environments are provided in this section.
For all of the above environments, we
regulate the `speed' of non-stationarity to characterize an algorithms’ ability to adapt.
Higher speed corresponds to a faster rate of non-stationarity; A speed of zero
indicates that the environment is stationary.

\paragraph{RoboToy: }
This domain corresponds to the toy robot scenario depicted in Figure \ref{fig:badpassive}.
Here, a robot can accomplish a task using either by `running' or `walking'.
Robot finishes a task faster when `running' than `walking' and thus the reward received at the end of 'running' is higher.
However, `running' causes more wear and tear on the robot, thereby degrading the performance of both `running' or `walking' in the future.
Since the past interactions influence the non-stationarity, this is an instance of active non-stationarity.
%

To perform more ablations on our algorithms, we also simulated a \textbf{RoboToy-Passive} domain, where there is no active non-stationarity as above. Instead, the reward obtained at the end of executing the options `walking' or `running' fluctuate across episodes.
Therefore, the changes to the underlying system are independent of the actions taken by the agent in the past.

For both the active and passive version of this domain, we collect data using a behavior policy that chooses `walking' more frequently, and the evaluation policy is designed such that it chooses `running' more frequently.

\paragraph{Non-stationary Mountain Car: }
 In real-world mechanical systems, motors undergo wear and tear over time based on how vigorously they have been used in the past.
 To simulate similar performance degradation, we adapt the classic (stationary) mountain car domain \citep{moore1990efficient}.
  We modify the domain such that at every episode the effective acceleration force is decayed proportional to the average velocity of the car in the previous episode.
 This results in active non-stationarity as the change in the system is based on the actions taken by the agent in the past.
 Similar to the works by \citep{thomas2015safe,jiang2015doubly}, we make use of macro-actions to repeat an action 10 times, which helps in reducing the effective horizon length of each episode.
 The maximum number of step per episode using these macros is 30.
 
For our experiments, using an actor-critic algorithm \citep{SuttonBarto2} we find a near-optimal policy $\pi$ on the stationary version of the mountain car domain, which we use as the evaluation policy.
Let $\pi^\texttt{rand}$ be a random policy with uniform distribution over the actions.
Then we define the behavior policy $\beta(o,a) \coloneqq 0.5 \pi(o,a) + 0.5 \pi^\texttt{rand}(o,a)$ for all states and actions.

\paragraph{Type-1 Diabetes Management: }
Automated healthcare systems that aim to personalise for individual patients should account for the physiological changes of the patient over time.
To simulate such a scenario we use an open-source implementation \citep{simglucose} of the U.S. Food and Drug Administration (FDA) approved Type-1 Diabetes Mellitus simulator (T1DMS) \citep{man2014uva} for the treatment of Type-1 diabetes,
where we induced non-stationarity by oscillating the body parameters (e.g., rate of glucose absorption, insulin sensitivity, etc.) between two known configurations available in the simulator. 
This induces passive non-stationarity, that is, changes are not dependent on past actions.

Each step of an episode corresponds to a minute ($1440$ timesteps--one for each minute in a day) in an \textit{in-silico} patient's body and state transitions are governed by a continuous time non-linear ordinary differential equation (ODE) \citep{man2014uva}.
This makes the problem particularly challenging as it is unclear how the performance trends of policies vary in this domain when the physiological parameters of the patient are changed. 
Notice that as the parameters that are being oscillated are inputs to a non-linear ODE system, the exact trend of performance for any policy is unknown.
This more closely reflects a real-world setting where \thref{ass:fixedf} might not hold, as every policy's performance trend in real-world problems cannot be expected to follow \textit{any} specific trend \textit{exactly}--one can only hope to obtain a coarse approximation of the trend.

For our experiments, using an actor-critic algorithm \citep{SuttonBarto2} we find a near-optimal policy $\pi$ on the stationary version of this domain, which we use as the evaluation policy.
The policy learns the CR and CF parameters of a basal-bolus controller.
Let $\pi^\texttt{rand}$ be a random policy with uniform distribution over actions.
Then we define the behavior policy $\beta(o,a) \coloneqq 0.5 \pi(o,a) + 0.5 \pi^\texttt{rand}(o,a)$ for all states and actions.

\paragraph{MEDEVAC: } This domain stands for \textit{med}ical \textit{evac}uation using air ambulances. This domain was developed by \citet{robbins2020approximate} for optimally routing air ambulances to provide medical assistance in regions of conflict.
This domain divides the region of conflict into $34$ mutually exclusive zones, and has $4$ air ambulances to serve all zones when an event occurs.
Based on real-data, this domain simulates the arrival of different events, from different zones, where each event can have $3$ different priority levels.
Serving higher priority events yields higher rewards.
If an ambulance is assigned to an event, it will finish the assignment in a time dependent on the distance between the base of the ambulance and the zone of the corresponding event.
While engaged in an assignment, that ambulance is no longer available to serve other events.
A good controller decides whether to deploy, and which MEDEVAC to deploy, to serve any event (at the risk of not being able to serve a new high-priority event if all ambulances become occupied).

The original implementation of the domain assumes that the arrival rates of the events and the time taken by an ambulance to complete an event follow a Poisson process with a fixed rate.
However, in reality, the arrival rates of different events can change based on external incidents during conflict. Similarly, the completion rate can also change based on how frequently an ambulance is deployed.
To simulate such non-stationarity, we oscillate the arrival rate of the incoming high-priority events, which induces passive non-stationarity.
Further, to induce wear and tear, we slowly decay the rate at which an ambulance can finish an assignment. This decay is proportional to how frequently the ambulance was used in the past.
This induces active non-stationarity. 
The presence of both active and passive changes makes this domain subject to hybrid non-stationarity.

Similar to other domains, we used an actor-critic algorithm \citep{SuttonBarto2} we find a near-optimal policy $\pi$ on the stationary version of this domain, which we use as the evaluation policy.
Let $\pi^\texttt{rand}$ be a random policy with uniform distribution over the actions.
Then we define the behavior policy $\beta(o,a) \coloneqq 0.5 \pi(o,a) + 0.5 \pi^\texttt{rand}(o,a)$ for all states and actions.

\subsection{Additional Results}

\label{sec:passdisc}
While the primary focus of this chapter was to develop methods to handle active/hybrid non-stationarity, we observed that the proposed method OPEN also provides benefits over the earlier algorithm Pro-WLS even when it is known that there is only passive non-stationarity in the environment.

\subsubsection{Single Run}
Similar to Figure \ref{fig:activestep}, in Figure \ref{fig:passivestep}  we present a step by step breakdown of the intermediate stages of a single run of OPEN on the RoboToy-Passive domain.
Here the trend in how the performance of the evaluation policy was changing in the past remains the same in the future.
When only passive non-stationarity is present, the double counter-factual correction performed by OPEN is superfluous.
However,  it can be observed that OPEN can still correctly identify the trend and provide useful predictions of $\pi$'s future performance.

\begin{figure}
    \centering
    
    \begin{minipage}{.4\textwidth}
    \centering
     \includegraphics[width=\textwidth]{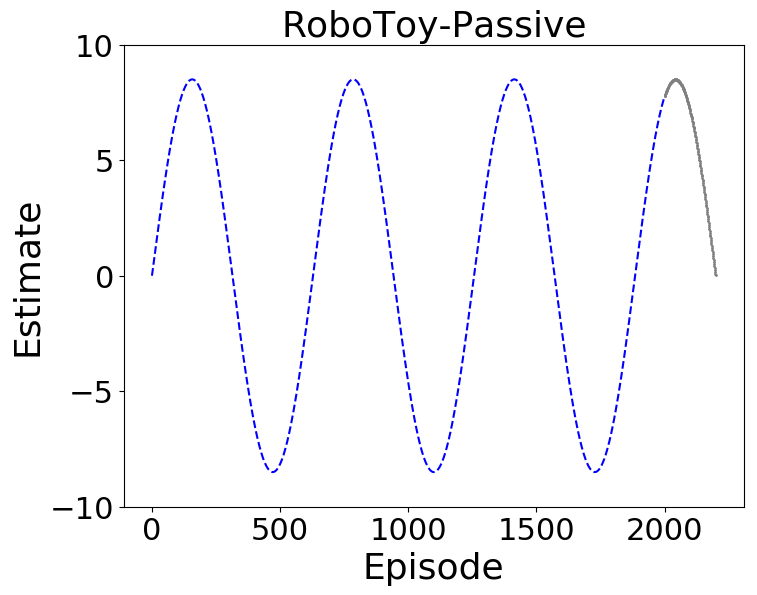}
    \end{minipage}
    \begin{minipage}{.58\textwidth}
    { The blue curve corresponds to the performances $J_i(\pi)$ for the past episodes.
   As there is no active non-stationarity, the choice of actions executed does not impact the underlying non-stationarity.
   Therefore, $J_i(\pi)$ follows the same trend in future as it did in the past.
   The blue and gray curves are unknown to the algorithm.
    }
      \end{minipage}
      \\
    \begin{minipage}{.4\textwidth}
    \centering
     \includegraphics[width=\textwidth]{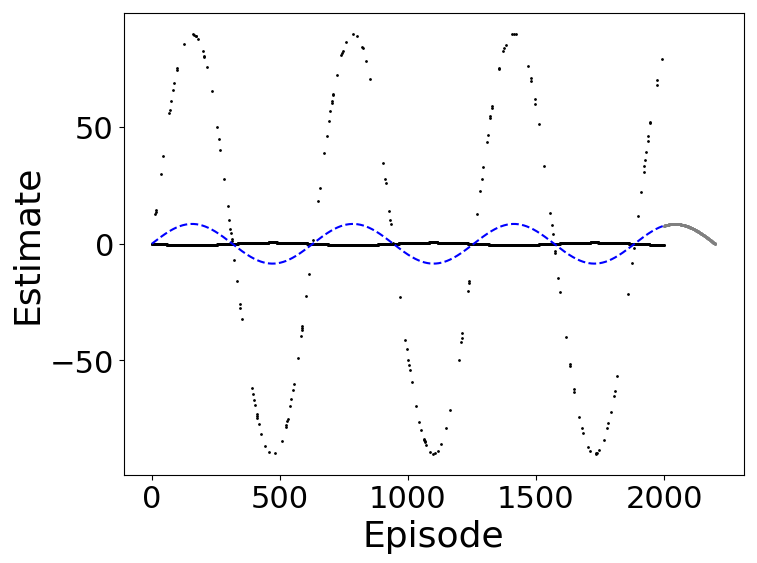}
    \end{minipage}
    \begin{minipage}{.58\textwidth}
    { OPEN first uses historical data to obtain counterfactual estimates of $J_i(\pi)$ for the past episodes.
    One can see the high-variance in these estimates (notice the change in the y-scale) due to the use of importance sampling.
    }
      \end{minipage}
      \\
    \begin{minipage}{.4\textwidth}
    \centering
     \includegraphics[width=\textwidth]{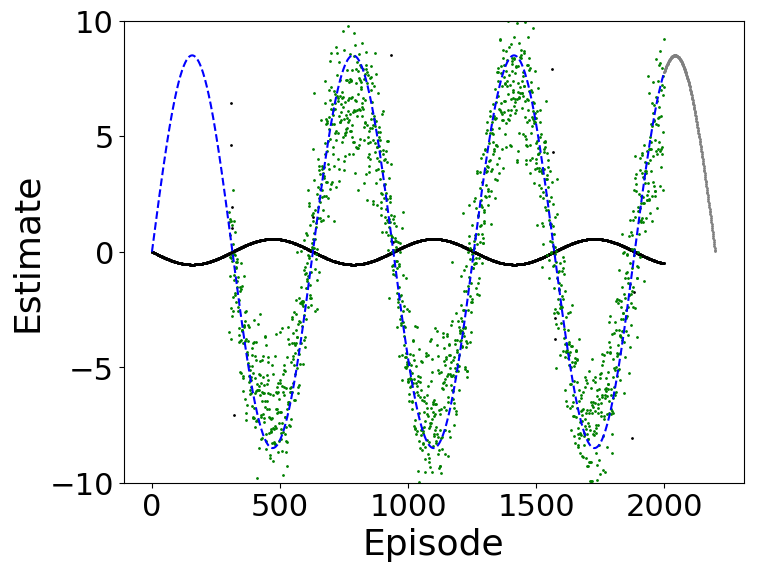}
    \end{minipage}
    \begin{minipage}{.58\textwidth}
    { Before naively auto-regressing, OPEN first aims to denoise the past performance estimates using the first stage of instrument variable regression. Since $p=300$, the first $300$ terms were not denoised. It can be observed that OPEN successfully denoises the importance sampling estimates.
    }
      \end{minipage}
      \\
    \begin{minipage}{.4\textwidth}
    \centering
     \includegraphics[width=\textwidth]{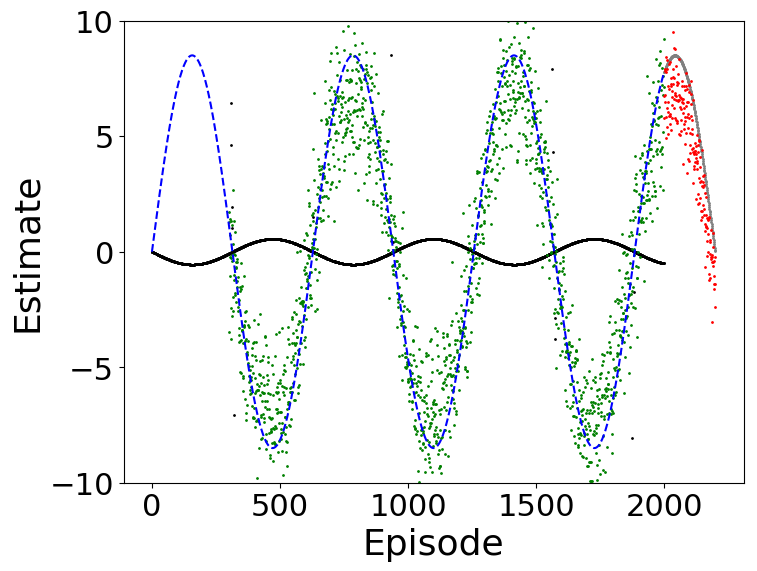}
    \end{minipage}
    \begin{minipage}{.58\textwidth}
    {  Using the denoised estimates of past performances, with the second use of counterfactual reasoning, OPEN performs the second stage of regression to forecast the future performance when $\pi$ will be deployed.
    In the passive setting, use of double-counterfactual is superfluous but OPEN is still able to correctly predict the future performance.
    }
      \end{minipage}
      \includegraphics[width=0.7\textwidth]{images/demo_legend.png}
      \caption{An illustrative step by step breakdown of the stages in the proposed algorithm OPEN for the RoboToy-Passive domain.}
    \label{fig:passivestep}     
\end{figure}

\subsubsection{Summary Plots}

In Figure \ref{fig:passive_plots} we provide bias and MSE analysis of different algorithms on the domains that exhibit passive non-stationarity.
Except for the stationary setting, where WIS has the best performance overall, we observe that for all other settings in the plot, OPEN performs better than both Pro-WLS and WIS consistently.

One thing that particularly stands out in these plots is the poor performance of Pro-WLS, despite being designed for the passive setting.
We observed that because of the choice of parametric regression using the Fourier basis, Pro-WLS tends to suffer from high bias when the number of Fourier terms is not sufficient to model the underlying trend.
Also, if the number of Fourier terms is increased naively, then they overfit the data and extrapolate poorly, thereby resulting in high-variance.
In contrast, our method is based on an auto-regressive based time-series forecast that is more robust to the model choice (we kept the number of lag terms for auto-regression as $p=300$ for OPEN for all our experiments).

To obtain all the results for Figure \ref{fig:active_plots_bias} and Figure \ref{fig:passive_plots},
in total 30 different seeds were used for each speed of each domain for each algorithm to get the standard error.
The authors had shared access to a computing cluster, consisting of 50 compute nodes with 28 cores each,
which was used to run all the experiments.

\begin{figure}
    \centering
    \includegraphics[width=0.35\textwidth]{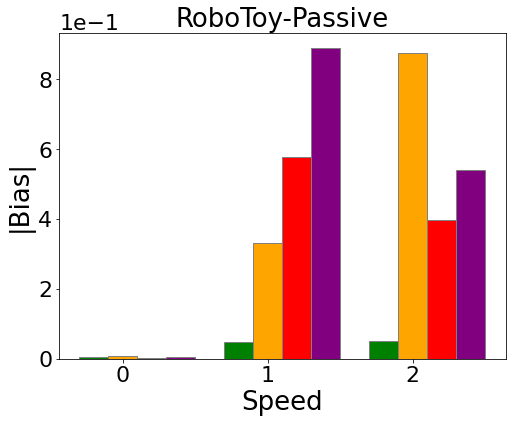}
    \includegraphics[width=0.35\textwidth]{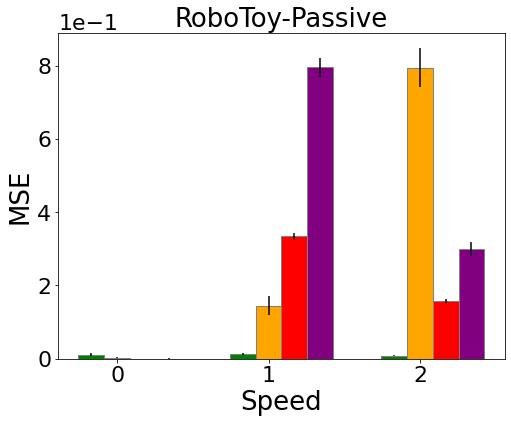}
    \\
    \includegraphics[width=0.6\textwidth]{images/MSE_legend.png}
    \caption{Comparison of different algorithms for predicting the future performance of evaluation policy $\pi$ on domains that exhibit passive non-stationarity.
      On the x-axis is the speed, which corresponds to the rate of non-stationarity; higher speed indicates a faster rate of change and a speed of zero indicates a stationary domain. 
      \textbf{(TOP)} On the y-axis is the absolute bias in the performance estimate.       \textbf{(Bottom)} On the y-axis is the mean squared error (MSE) in the performance estimate. \textbf{Lower is better} for all of these plots.
      For each domain, for each speed, for each algorithm, 30 trials were executed.
      Discussion of these plots can be found in Section \ref{sec:passdisc}.}
    \label{fig:passive_plots}
\end{figure}

\subsection{Ablation Study}
In this section we study the sensitivity to hyper-parameters for the proposed method OPEN and the baseline method Pro-WLS \citep{chandak2020optimizing}.
The hyper-parameter for OPEN corresponds to the number of past terms to condition on for auto-regression, as discussed in \thref{rem:p}.
The hyper-parameter for Pro-WLS corresponds to the order of Fourier bases required for parametric regression.
In Figure \ref{fig:ablation} we present the results for how the performance of the methods vary for different choices of hyper-parameters.

\label{apx:ablation}

\begin{figure}
    \centering
    \includegraphics[width=0.32\textwidth]{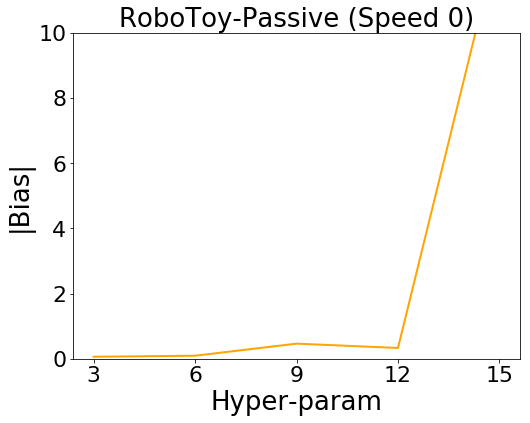}
    \includegraphics[width=0.32\textwidth]{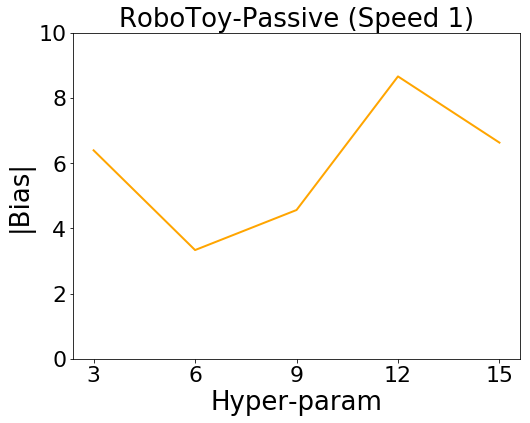}
    \includegraphics[width=0.32\textwidth]{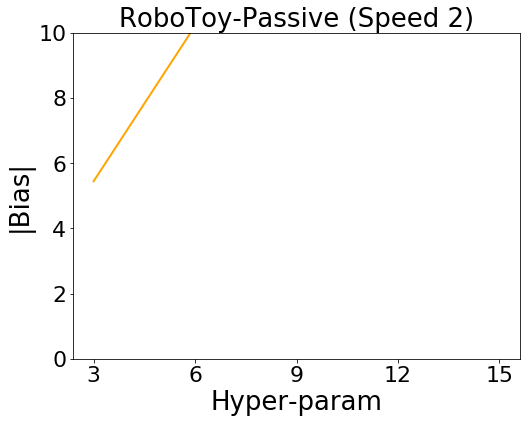}
    \\
    \includegraphics[width=0.32\textwidth]{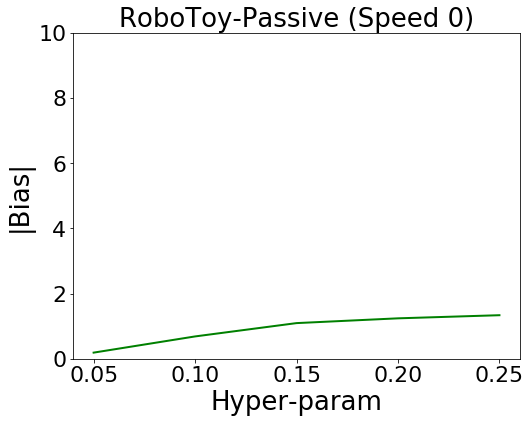}
    \includegraphics[width=0.32\textwidth]{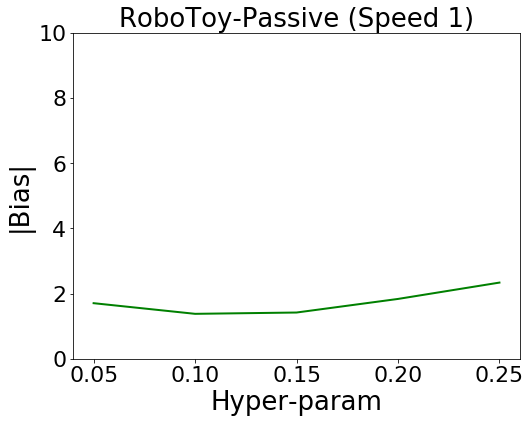}
    \includegraphics[width=0.32\textwidth]{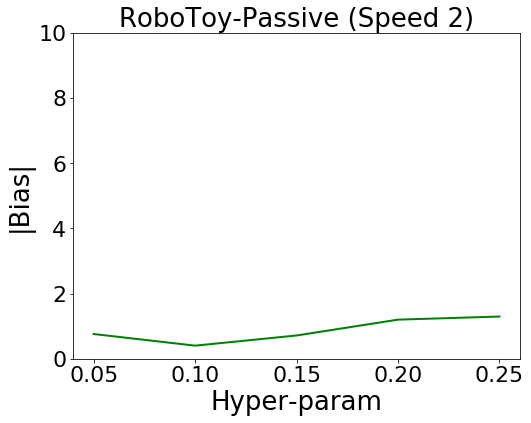}
    \caption{\textbf{(Top)} Absolute bias in prediction of Pro-WLS for different choices of its hyper-parameter. \textbf{(Bottom)}  Absolute bias in prediction of OPEN for different choices of its hyper-parameter. For all the plots, lower value is better.
    Overall, we observe that OPEN being an auto-regressive method can extrapolate/forecast better and is thus more robust to hyper-parameters than Pro-WLS that uses Fourier bases for regression and is not as good for extrapolation.}
    \label{fig:ablation}
\end{figure}

%

%% file: main.bbl
\begin{thebibliography}{101}
\providecommand{\natexlab}[1]{#1}
\providecommand{\url}[1]{\texttt{#1}}
\expandafter\ifx\csname urlstyle\endcsname\relax
  \providecommand{\doi}[1]{doi: #1}\else
  \providecommand{\doi}{doi: \begingroup \urlstyle{rm}\Url}\fi

\bibitem[Abbott(2007)]{abbottIV}
M.~Abbott.
\newblock Instrumental variables (iv) estimation: An introduction, 2007.
\newblock
  \url{http://qed.econ.queensu.ca/pub/faculty/abbott/econ481/481note09_f07.pdf}.

\bibitem[Achen(2000)]{achen2000lagged}
C.~H. Achen.
\newblock Why lagged dependent variables can suppress the explanatory power of
  other independent variables.
\newblock In \emph{annual meeting of the political methodology section of the
  American political science association, UCLA}, volume~20, pages 7--2000,
  2000.

\bibitem[Agarwal et~al.(2020)Agarwal, Henaff, Kakade, and Sun]{agarwal2020pc}
A.~Agarwal, M.~Henaff, S.~Kakade, and W.~Sun.
\newblock Pc-pg: Policy cover directed exploration for provable policy gradient
  learning.
\newblock \emph{Advances in Neural Information Processing Systems},
  33:\penalty0 13399--13412, 2020.

\bibitem[Alegre et~al.(2021)Alegre, Bazzan, and da~Silva]{alegre2021minimum}
L.~N. Alegre, A.~L. Bazzan, and B.~C. da~Silva.
\newblock Minimum-delay adaptation in non-stationary reinforcement learning via
  online high-confidence change-point detection.
\newblock \emph{arXiv preprint arXiv:2105.09452}, 2021.

\bibitem[Ammar et~al.(2015)Ammar, Tutunov, and Eaton]{ammar2015safe}
H.~B. Ammar, R.~Tutunov, and E.~Eaton.
\newblock Safe policy search for lifelong reinforcement learning with sublinear
  regret.
\newblock In \emph{International Conference on Machine Learning}, pages
  2361--2369. PMLR, 2015.

\bibitem[Basso and Engel(2009)]{basso2009reinforcement}
E.~W. Basso and P.~M. Engel.
\newblock Reinforcement learning in non-stationary continuous time and space
  scenarios.
\newblock In \emph{Artificial Intelligence National Meeting}, volume~7, pages
  1--8. Citeseer, 2009.

\bibitem[Bellemare et~al.(2017)Bellemare, Masaki, and
  Pepinsky]{bellemare2017lagged}
M.~F. Bellemare, T.~Masaki, and T.~B. Pepinsky.
\newblock Lagged explanatory variables and the estimation of causal effect.
\newblock \emph{The Journal of Politics}, 79\penalty0 (3):\penalty0 949--963,
  2017.

\bibitem[Bennett et~al.(2019)Bennett, Kallus, and Schnabel]{bennett2019deep}
A.~Bennett, N.~Kallus, and T.~Schnabel.
\newblock Deep generalized method of moments for instrumental variable
  analysis.
\newblock \emph{arXiv preprint arXiv:1905.12495}, 2019.

\bibitem[Bennett et~al.(2021)Bennett, Kallus, Li, and Mousavi]{bennett2021off}
A.~Bennett, N.~Kallus, L.~Li, and A.~Mousavi.
\newblock Off-policy evaluation in infinite-horizon reinforcement learning with
  latent confounders.
\newblock In \emph{International Conference on Artificial Intelligence and
  Statistics}, pages 1999--2007. PMLR, 2021.

\bibitem[Besbes et~al.(2014)Besbes, Gur, and Zeevi]{besbes2014stochastic}
O.~Besbes, Y.~Gur, and A.~Zeevi.
\newblock Stochastic multi-armed-bandit problem with non-stationary rewards.
\newblock In \emph{Advances in {N}eural {I}nformation {P}rocessing {S}ystems},
  pages 199--207, 2014.

\bibitem[Bowling(2005)]{bowling2005convergence}
M.~Bowling.
\newblock Convergence and no-regret in multiagent learning.
\newblock In \emph{Advances in {N}eural {I}nformation {P}rocessing {S}ystems},
  pages 209--216, 2005.

\bibitem[Boyan(1999)]{boyan1999least}
J.~A. Boyan.
\newblock Least-squares temporal difference learning.
\newblock In \emph{ICML}, pages 49--56, 1999.

\bibitem[Buckman et~al.(2020)Buckman, Gelada, and
  Bellemare]{buckman2020importance}
J.~Buckman, C.~Gelada, and M.~G. Bellemare.
\newblock The importance of pessimism in fixed-dataset policy optimization.
\newblock \emph{arXiv preprint arXiv:2009.06799}, 2020.

\bibitem[Cameron(2019)]{cameronIV}
A.~C. Cameron.
\newblock Instrument variables, 2019.
\newblock \url{http://cameron.econ.ucdavis.edu/e240a/ch04iv.pdf}.

\bibitem[Cetin and Celiktutan(2021)]{cetin2021learning}
E.~Cetin and O.~Celiktutan.
\newblock Learning pessimism for robust and efficient off-policy reinforcement
  learning.
\newblock \emph{arXiv preprint arXiv:2110.03375}, 2021.

\bibitem[Chandak et~al.(2020{\natexlab{a}})Chandak, Jordan, Theocharous, White,
  and Thomas]{chandak2020towards}
Y.~Chandak, S.~M. Jordan, G.~Theocharous, M.~White, and P.~S. Thomas.
\newblock Towards safe policy improvement for non-stationary mdps.
\newblock \emph{Neural Information Processing Systems}, 2020{\natexlab{a}}.

\bibitem[Chandak et~al.(2020{\natexlab{b}})Chandak, Theocharous, Shankar,
  Mahadevan, White, and Thomas]{chandak2020optimizing}
Y.~Chandak, G.~Theocharous, S.~Shankar, S.~Mahadevan, M.~White, and P.~S.
  Thomas.
\newblock Optimizing for the future in non-stationary mdps.
\newblock \emph{International Conference on Machine Learning},
  2020{\natexlab{b}}.

\bibitem[Chandak et~al.(2021)Chandak, Niekum, da~Silva, Learned-Miller,
  Brunskill, and Thomas]{chandak2021universal}
Y.~Chandak, S.~Niekum, B.~da~Silva, E.~Learned-Miller, E.~Brunskill, and P.~S.
  Thomas.
\newblock Universal off-policy evaluation.
\newblock \emph{Advances in Neural Information Processing Systems}, 34, 2021.

\bibitem[Chandra(1991)]{chandra1991extensions}
T.~K. Chandra.
\newblock Extensions of rajchman's strong law of large numbers.
\newblock \emph{Sankhy{\=a}: The Indian Journal of Statistics, Series A}, pages
  118--121, 1991.

\bibitem[Choi et~al.(2000)Choi, Yeung, and Zhang]{choi2000environment}
S.~P. Choi, D.-Y. Yeung, and N.~L. Zhang.
\newblock An environment model for nonstationary reinforcement learning.
\newblock In \emph{Advances in {N}eural {I}nformation {P}rocessing {S}ystems},
  pages 987--993, 2000.

\bibitem[Cinelli et~al.(2021)Cinelli, Morales, Galeazzi, Quattrociocchi, and
  Starnini]{cinelli2021echo}
M.~Cinelli, G.~D.~F. Morales, A.~Galeazzi, W.~Quattrociocchi, and M.~Starnini.
\newblock The echo chamber effect on social media.
\newblock \emph{Proceedings of the National Academy of Sciences}, 118\penalty0
  (9), 2021.

\bibitem[Conitzer and Sandholm(2007)]{conitzer2007awesome}
V.~Conitzer and T.~Sandholm.
\newblock Awesome: A general multiagent learning algorithm that converges in
  self-play and learns a best response against stationary opponents.
\newblock \emph{Machine Learning}, 67\penalty0 (1-2):\penalty0 23--43, 2007.

\bibitem[Cox and Miller(2017)]{cox2017theory}
D.~R. Cox and H.~D. Miller.
\newblock \emph{The theory of stochastic processes}.
\newblock Routledge, 2017.

\bibitem[Da~Silva et~al.(2006)Da~Silva, Basso, Bazzan, and
  Engel]{da2006dealing}
B.~C. Da~Silva, E.~W. Basso, A.~L. Bazzan, and P.~M. Engel.
\newblock Dealing with non-stationary environments using context detection.
\newblock In \emph{Proceedings of the 23rd international conference on Machine
  learning}, pages 217--224, 2006.

\bibitem[Dai et~al.(2020)Dai, Nachum, Chow, Li, Szepesv{\'a}ri, and
  Schuurmans]{dai2020coindice}
B.~Dai, O.~Nachum, Y.~Chow, L.~Li, C.~Szepesv{\'a}ri, and D.~Schuurmans.
\newblock Coindice: Off-policy confidence interval estimation.
\newblock \emph{arXiv preprint arXiv:2010.11652}, 2020.

\bibitem[Doshi-Velez and Konidaris(2016)]{doshi2016hidden}
F.~Doshi-Velez and G.~Konidaris.
\newblock Hidden parameter markov decision processes: A semiparametric
  regression approach for discovering latent task parametrizations.
\newblock In \emph{IJCAI: proceedings of the conference}, volume 2016, page
  1432. NIH Public Access, 2016.

\bibitem[Dulac-Arnold et~al.(2019)Dulac-Arnold, Mankowitz, and
  Hester]{dulac2019challenges}
G.~Dulac-Arnold, D.~Mankowitz, and T.~Hester.
\newblock Challenges of real-world reinforcement learning.
\newblock \emph{arXiv preprint arXiv:1904.12901}, 2019.

\bibitem[Espeholt et~al.(2018)Espeholt, Soyer, Munos, Simonyan, Mnih, Ward,
  Doron, Firoiu, Harley, Dunning, et~al.]{espeholt2018impala}
L.~Espeholt, H.~Soyer, R.~Munos, K.~Simonyan, V.~Mnih, T.~Ward, Y.~Doron,
  V.~Firoiu, T.~Harley, I.~Dunning, et~al.
\newblock Impala: Scalable distributed deep-rl with importance weighted
  actor-learner architectures.
\newblock In \emph{International conference on machine learning}, pages
  1407--1416. PMLR, 2018.

\bibitem[Feng et~al.(2021)Feng, Tang, na~zhang, and qiang
  liu]{feng2021nonasymptotic}
Y.~Feng, Z.~Tang, na~zhang, and qiang liu.
\newblock Non-asymptotic confidence intervals of off-policy evaluation: Primal
  and dual bounds.
\newblock In \emph{International Conference on Learning Representations}, 2021.
\newblock URL \url{https://openreview.net/forum?id=dKg5D1Z1Lm}.

\bibitem[Foerster et~al.(2018)Foerster, Chen, Al-Shedivat, Whiteson, Abbeel,
  and Mordatch]{foerster2018learning}
J.~Foerster, R.~Y. Chen, M.~Al-Shedivat, S.~Whiteson, P.~Abbeel, and
  I.~Mordatch.
\newblock Learning with opponent-learning awareness.
\newblock In \emph{Proceedings of the 17th International Conference on
  Autonomous Agents and MultiAgent Systems}, pages 122--130. International
  Foundation for Autonomous Agents and Multiagent Systems, 2018.

\bibitem[Foster et~al.(2016)Foster, Li, Lykouris, Sridharan, and
  Tardos]{foster2016learning}
D.~J. Foster, Z.~Li, T.~Lykouris, K.~Sridharan, and E.~Tardos.
\newblock Learning in games: Robustness of fast convergence.
\newblock In \emph{Advances in {N}eural {I}nformation {P}rocessing {S}ystems},
  pages 4734--4742, 2016.

\bibitem[Gemp and Mahadevan(2017)]{gemp2017online}
I.~Gemp and S.~Mahadevan.
\newblock Online monotone games.
\newblock \emph{arXiv preprint arXiv:1710.07328}, 2017.

\bibitem[Gillani et~al.(2018)Gillani, Yuan, Saveski, Vosoughi, and
  Roy]{gillani2018me}
N.~Gillani, A.~Yuan, M.~Saveski, S.~Vosoughi, and D.~Roy.
\newblock Me, my echo chamber, and i: introspection on social media
  polarization.
\newblock In \emph{Proceedings of the 2018 World Wide Web Conference}, pages
  823--831, 2018.

\bibitem[Hamilton(1994)]{hamilton1994state}
J.~D. Hamilton.
\newblock State-space models.
\newblock \emph{Handbook of econometrics}, 4:\penalty0 3039--3080, 1994.

\bibitem[Hartford et~al.(2017)Hartford, Lewis, Leyton-Brown, and
  Taddy]{hartford2017deep}
J.~Hartford, G.~Lewis, K.~Leyton-Brown, and M.~Taddy.
\newblock Deep iv: A flexible approach for counterfactual prediction.
\newblock In \emph{International Conference on Machine Learning}, pages
  1414--1423. PMLR, 2017.

\bibitem[Harutyunyan et~al.(2016)Harutyunyan, Bellemare, Stepleton, and
  Munos]{harutyunyan2016q}
A.~Harutyunyan, M.~G. Bellemare, T.~Stepleton, and R.~Munos.
\newblock Q ($\lambda$ ) with off-policy corrections.
\newblock In \emph{International Conference on Algorithmic Learning Theory},
  pages 305--320. Springer, 2016.

\bibitem[Hennes et~al.(2019)Hennes, Morrill, Omidshafiei, Munos, Perolat,
  Lanctot, Gruslys, Lespiau, Parmas, Duenez-Guzman, et~al.]{hennes2019neural}
D.~Hennes, D.~Morrill, S.~Omidshafiei, R.~Munos, J.~Perolat, M.~Lanctot,
  A.~Gruslys, J.-B. Lespiau, P.~Parmas, E.~Duenez-Guzman, et~al.
\newblock Neural replicator dynamics.
\newblock \emph{arXiv preprint arXiv:1906.00190}, 2019.

\bibitem[Hochreiter and Schmidhuber(1997)]{hochreiter1997long}
S.~Hochreiter and J.~Schmidhuber.
\newblock Long short-term memory.
\newblock \emph{Neural computation}, 9\penalty0 (8):\penalty0 1735--1780, 1997.

\bibitem[Jagerman et~al.(2019)Jagerman, {M}arkov, and
  de~Rijke]{jagerman2019when}
R.~Jagerman, I.~{M}arkov, and M.~de~Rijke.
\newblock When people change their mind: Off-policy evaluation in
  non-stationary recommendation environments.
\newblock In \emph{Proceedings of the Twelfth {ACM} International Conference on
  Web Search and Data Mining, Melbourne, VIC, Australia, February 11-15, 2019},
  2019.

\bibitem[Jaques et~al.(2019)Jaques, Lazaridou, Hughes, Gulcehre, Ortega,
  Strouse, Leibo, and De~Freitas]{jaques2019social}
N.~Jaques, A.~Lazaridou, E.~Hughes, C.~Gulcehre, P.~Ortega, D.~Strouse, J.~Z.
  Leibo, and N.~De~Freitas.
\newblock Social influence as intrinsic motivation for multi-agent deep
  reinforcement learning.
\newblock In \emph{International Conference on Machine Learning}, pages
  3040--3049. PMLR, 2019.

\bibitem[Jiang and Huang(2020)]{jiang2020minimax}
N.~Jiang and J.~Huang.
\newblock Minimax confidence interval for off-policy evaluation and policy
  optimization.
\newblock \emph{arXiv preprint arXiv:2002.02081}, 2020.

\bibitem[Jiang and Li(2015)]{jiang2015doubly}
N.~Jiang and L.~Li.
\newblock Doubly robust off-policy value evaluation for reinforcement learning.
\newblock \emph{arXiv preprint arXiv:1511.03722}, 2015.

\bibitem[Khetarpal et~al.(2020)Khetarpal, Riemer, Rish, and
  Precup]{khetarpal2020towards}
K.~Khetarpal, M.~Riemer, I.~Rish, and D.~Precup.
\newblock Towards continual reinforcement learning: A review and perspectives.
\newblock \emph{arXiv preprint arXiv:2012.13490}, 2020.

\bibitem[Levine et~al.(2017)Levine, Crammer, and Mannor]{levine2017rotting}
N.~Levine, K.~Crammer, and S.~Mannor.
\newblock Rotting bandits.
\newblock In \emph{Advances in {N}eural {I}nformation {P}rocessing {S}ystems},
  pages 3074--3083, 2017.

\bibitem[Li and de~Rijke(2019)]{li2019cascading}
C.~Li and M.~de~Rijke.
\newblock Cascading non-stationary bandits: Online learning to rank in the
  non-stationary cascade model.
\newblock \emph{arXiv preprint arXiv:1905.12370}, 2019.

\bibitem[Liotet et~al.(2021)Liotet, Vidaich, Metelli, and
  Restelli]{liotet2021lifelong}
P.~Liotet, F.~Vidaich, A.~M. Metelli, and M.~Restelli.
\newblock Lifelong hyper-policy optimization with multiple importance sampling
  regularization.
\newblock \emph{arXiv preprint arXiv:2112.06625}, 2021.

\bibitem[Liu et~al.(2018)Liu, Li, Tang, and Zhou]{liu2018breaking}
Q.~Liu, L.~Li, Z.~Tang, and D.~Zhou.
\newblock Breaking the curse of horizon: Infinite-horizon off-policy
  estimation.
\newblock In \emph{Advances in Neural Information Processing Systems}, pages
  5356--5366, 2018.

\bibitem[Liu et~al.(2020)Liu, Shang, and Cheng]{liu2020deep}
R.~Liu, Z.~Shang, and G.~Cheng.
\newblock On deep instrumental variables estimate.
\newblock \emph{arXiv preprint arXiv:2004.14954}, 2020.

\bibitem[Mahmood et~al.(2014)Mahmood, Van~Hasselt, and
  Sutton]{mahmood2014weighted}
A.~R. Mahmood, H.~Van~Hasselt, and R.~S. Sutton.
\newblock Weighted importance sampling for off-policy learning with linear
  function approximation.
\newblock In \emph{NIPS}, pages 3014--3022, 2014.

\bibitem[Mahmood et~al.(2015)Mahmood, Yu, White, and
  Sutton]{mahmood2015emphatic}
A.~R. Mahmood, H.~Yu, M.~White, and R.~S. Sutton.
\newblock Emphatic temporal-difference learning.
\newblock \emph{arXiv preprint arXiv:1507.01569}, 2015.

\bibitem[Man et~al.(2014)Man, Micheletto, Lv, Breton, Kovatchev, and
  Cobelli]{man2014uva}
C.~D. Man, F.~Micheletto, D.~Lv, M.~Breton, B.~Kovatchev, and C.~Cobelli.
\newblock The {UVA/PADOVA} type 1 diabetes simulator: {N}ew features.
\newblock \emph{Journal of Diabetes Science and Technology}, 8\penalty0
  (1):\penalty0 26--34, 2014.

\bibitem[Mealing and Shapiro(2013)]{mealing2013opponent}
R.~Mealing and J.~L. Shapiro.
\newblock Opponent modelling by sequence prediction and lookahead in two-player
  games.
\newblock In \emph{International Conference on Artificial Intelligence and Soft
  Computing}, pages 385--396. Springer, 2013.

\bibitem[Moore(1990)]{moore1990efficient}
A.~W. Moore.
\newblock Efficient memory-based learning for robot control.
\newblock 1990.

\bibitem[Moulines(2008)]{moulines2008}
E.~Moulines.
\newblock On upper-confidence bound policies for non-stationary bandit
  problems.
\newblock \emph{arXiv preprint arXiv:0805.3415}, 2008.

\bibitem[Munos et~al.(2016)Munos, Stepleton, Harutyunyan, and
  Bellemare]{munos2016safe}
R.~Munos, T.~Stepleton, A.~Harutyunyan, and M.~Bellemare.
\newblock Safe and efficient off-policy reinforcement learning.
\newblock \emph{Advances in neural information processing systems}, 29, 2016.

\bibitem[Nachum and Dai(2020)]{nachum2020reinforcement}
O.~Nachum and B.~Dai.
\newblock Reinforcement learning via fenchel-rockafellar duality.
\newblock \emph{arXiv preprint arXiv:2001.01866}, 2020.

\bibitem[Nachum et~al.(2019)Nachum, Chow, Dai, and Li]{nachum2019dualdice}
O.~Nachum, Y.~Chow, B.~Dai, and L.~Li.
\newblock Dualdice: Behavior-agnostic estimation of discounted stationary
  distribution corrections.
\newblock \emph{Advances in Neural Information Processing Systems}, 32, 2019.

\bibitem[Namkoong et~al.(2020)Namkoong, Keramati, Yadlowsky, and
  Brunskill]{namkoong2020off}
H.~Namkoong, R.~Keramati, S.~Yadlowsky, and E.~Brunskill.
\newblock Off-policy policy evaluation for sequential decisions under
  unobserved confounding.
\newblock \emph{Advances in Neural Information Processing Systems},
  33:\penalty0 18819--18831, 2020.

\bibitem[Padakandla(2020)]{padakandla2020survey}
S.~Padakandla.
\newblock A survey of reinforcement learning algorithms for dynamically varying
  environments.
\newblock \emph{arXiv preprint arXiv:2005.10619}, 2020.

\bibitem[Padakandla et~al.(2019)Padakandla, J., and
  Bhatnagar]{padakandla2019reinforcement}
S.~Padakandla, P.~K. J., and S.~Bhatnagar.
\newblock Reinforcement learning in non-stationary environments.
\newblock \emph{CoRR}, abs/1905.03970, 2019.

\bibitem[Parker(2020)]{parkerIV}
J.~A. Parker.
\newblock Endogenous regressors and instrumental variables, 2020.
\newblock \url{https://www.reed.edu/economics/parker/312/notes/Notes11.pdf}.

\bibitem[Pearl et~al.(2000)]{pearl2000models}
J.~Pearl et~al.
\newblock Models, reasoning and inference.
\newblock \emph{Cambridge, UK: CambridgeUniversityPress}, 19, 2000.

\bibitem[Poiani et~al.(2021)Poiani, Tirinzoni, and Restelli]{poiani2021meta}
R.~Poiani, A.~Tirinzoni, and M.~Restelli.
\newblock Meta-reinforcement learning by tracking task non-stationarity.
\newblock \emph{arXiv preprint arXiv:2105.08834}, 2021.

\bibitem[Precup(2000)]{precup2000eligibility}
D.~Precup.
\newblock Eligibility traces for off-policy policy evaluation.
\newblock \emph{Computer Science Department Faculty Publication Series},
  page~80, 2000.

\bibitem[Puterman(1990)]{puterman1990markov}
M.~L. Puterman.
\newblock Markov decision processes.
\newblock \emph{Handbooks in operations research and management science},
  2:\penalty0 331--434, 1990.

\bibitem[Rachelson et~al.(2009)Rachelson, Fabiani, and
  Garcia]{rachelson2009timdppoly}
E.~Rachelson, P.~Fabiani, and F.~Garcia.
\newblock Timdppoly: An improved method for solving time-dependent mdps.
\newblock In \emph{2009 21st IEEE International Conference on Tools with
  Artificial Intelligence}, pages 796--799. IEEE, 2009.

\bibitem[Rajchman(1932)]{rajchman1932zaostrzone}
A.~Rajchman.
\newblock Zaostrzone prawo wielkich liczb.
\newblock \emph{Mathesis Polska}, 6:\penalty0 145--161, 1932.

\bibitem[Reed(2015)]{reed2015practice}
W.~R. Reed.
\newblock On the practice of lagging variables to avoid simultaneity.
\newblock \emph{Oxford Bulletin of Economics and Statistics}, 77\penalty0
  (6):\penalty0 897--905, 2015.

\bibitem[Robbins et~al.(2020)Robbins, Jenkins, Bastian, and
  Lunday]{robbins2020approximate}
M.~J. Robbins, P.~R. Jenkins, N.~D. Bastian, and B.~J. Lunday.
\newblock Approximate dynamic programming for the aeromedical evacuation
  dispatching problem: Value function approximation utilizing multiple level
  aggregation.
\newblock \emph{Omega}, 91:\penalty0 102020, 2020.

\bibitem[Russac et~al.(2019)Russac, Vernade, and Capp{\'e}]{russac2019weighted}
Y.~Russac, C.~Vernade, and O.~Capp{\'e}.
\newblock Weighted linear bandits for non-stationary environments.
\newblock \emph{Advances in Neural Information Processing Systems}, 32, 2019.

\bibitem[Seznec et~al.(2018)Seznec, Locatelli, Carpentier, Lazaric, and
  Valko]{seznec2018rotting}
J.~Seznec, A.~Locatelli, A.~Carpentier, A.~Lazaric, and M.~Valko.
\newblock Rotting bandits are no harder than stochastic ones.
\newblock \emph{arXiv preprint arXiv:1811.11043}, 2018.

\bibitem[Shi et~al.(2021)Shi, Uehara, and Jiang]{shi2021minimax}
C.~Shi, M.~Uehara, and N.~Jiang.
\newblock A minimax learning approach to off-policy evaluation in partially
  observable markov decision processes.
\newblock \emph{arXiv preprint arXiv:2111.06784}, 2021.

\bibitem[Singh et~al.(2000)Singh, Kearns, and Mansour]{singh2000nash}
S.~Singh, M.~Kearns, and Y.~Mansour.
\newblock Nash convergence of gradient dynamics in general-sum games.
\newblock In \emph{Proceedings of the Sixteenth conference on Uncertainty in
  artificial intelligence}, pages 541--548. Morgan Kaufmann Publishers Inc.,
  2000.

\bibitem[Sutton and Barto(2018)]{SuttonBarto2}
R.~S. Sutton and A.~G. Barto.
\newblock \emph{Reinforcement learning: An introduction}.
\newblock MIT Press, Cambridge, MA, 2 edition, 2018.

\bibitem[Sutton et~al.(2008)Sutton, Maei, and
  Szepesv{\'a}ri]{sutton2008convergent}
R.~S. Sutton, H.~Maei, and C.~Szepesv{\'a}ri.
\newblock A convergent $ o (n) $ temporal-difference algorithm for off-policy
  learning with linear function approximation.
\newblock \emph{Advances in neural information processing systems}, 21, 2008.

\bibitem[Sutton et~al.(2009)Sutton, Maei, Precup, Bhatnagar, Silver,
  Szepesv{\'a}ri, and Wiewiora]{sutton2009fast}
R.~S. Sutton, H.~R. Maei, D.~Precup, S.~Bhatnagar, D.~Silver,
  C.~Szepesv{\'a}ri, and E.~Wiewiora.
\newblock Fast gradient-descent methods for temporal-difference learning with
  linear function approximation.
\newblock In \emph{Proceedings of the 26th annual international conference on
  machine learning}, pages 993--1000, 2009.

\bibitem[Taiga et~al.(2021)Taiga, Fedus, Machado, Courville, and
  Bellemare]{taiga2021bonus}
A.~A. Taiga, W.~Fedus, M.~C. Machado, A.~Courville, and M.~G. Bellemare.
\newblock On bonus-based exploration methods in the arcade learning
  environment.
\newblock \emph{arXiv preprint arXiv:2109.11052}, 2021.

\bibitem[Tennenholtz et~al.(2020)Tennenholtz, Shalit, and
  Mannor]{tennenholtz2020off}
G.~Tennenholtz, U.~Shalit, and S.~Mannor.
\newblock Off-policy evaluation in partially observable environments.
\newblock In \emph{Proceedings of the AAAI Conference on Artificial
  Intelligence}, volume~34, pages 10276--10283, 2020.

\bibitem[Theocharous et~al.(2020)Theocharous, Chandak, Thomas, and
  de~Nijs]{theocharous2020reinforcement}
G.~Theocharous, Y.~Chandak, P.~S. Thomas, and F.~de~Nijs.
\newblock Reinforcement learning for strategic recommendations.
\newblock \emph{arXiv preprint arXiv:2009.07346}, 2020.

\bibitem[Thomas and Brunskill(2016)]{thomas2016data}
P.~Thomas and E.~Brunskill.
\newblock Data-efficient off-policy policy evaluation for reinforcement
  learning.
\newblock In \emph{International Conference on Machine Learning}, pages
  2139--2148, 2016.

\bibitem[Thomas et~al.(2015)Thomas, Theocharous, and
  Ghavamzadeh]{thomas2015higha}
P.~Thomas, G.~Theocharous, and M.~Ghavamzadeh.
\newblock High-confidence off-policy evaluation.
\newblock In \emph{Proceedings of the AAAI Conference on Artificial
  Intelligence}, volume~29, 2015.

\bibitem[Thomas(2015)]{thomas2015safe}
P.~S. Thomas.
\newblock \emph{Safe reinforcement learning}.
\newblock PhD thesis, University of Massachusetts Libraries, 2015.

\bibitem[Thomas et~al.(2017)Thomas, Theocharous, Ghavamzadeh, Durugkar, and
  Brunskill]{thomas2017predictive}
P.~S. Thomas, G.~Theocharous, M.~Ghavamzadeh, I.~Durugkar, and E.~Brunskill.
\newblock Predictive off-policy policy evaluation for nonstationary decision
  problems, with applications to digital marketing.
\newblock In \emph{AAAI}, pages 4740--4745, 2017.

\bibitem[Thomas et~al.(2019)Thomas, da~Silva, Barto, Giguere, Brun, and
  Brunskill]{thomas2019preventing}
P.~S. Thomas, B.~C. da~Silva, A.~G. Barto, S.~Giguere, Y.~Brun, and
  E.~Brunskill.
\newblock Preventing undesirable behavior of intelligent machines.
\newblock \emph{Science}, 366\penalty0 (6468):\penalty0 999--1004, 2019.

\bibitem[Uehara et~al.(2020)Uehara, Huang, and Jiang]{uehara2020minimax}
M.~Uehara, J.~Huang, and N.~Jiang.
\newblock Minimax weight and q-function learning for off-policy evaluation.
\newblock In \emph{International Conference on Machine Learning}, pages
  9659--9668. PMLR, 2020.

\bibitem[Vernade et~al.(2020)Vernade, Gyorgy, and Mann]{vernade2020non}
C.~Vernade, A.~Gyorgy, and T.~Mann.
\newblock Non-stationary delayed bandits with intermediate observations.
\newblock In \emph{International Conference on Machine Learning}, pages
  9722--9732. PMLR, 2020.

\bibitem[Wang et~al.(2019{\natexlab{a}})Wang, Zhou, Li, Varshney, and
  Zhao]{wang2019aware}
L.~Wang, H.~Zhou, B.~Li, L.~R. Varshney, and Z.~Zhao.
\newblock Be aware of non-stationarity: Nearly optimal algorithms for
  piecewise-stationary cascading bandits.
\newblock \emph{arXiv preprint arXiv:1909.05886}, 2019{\natexlab{a}}.

\bibitem[Wang et~al.(2007)Wang, Bowling, and Schuurmans]{wang2007dual}
T.~Wang, M.~Bowling, and D.~Schuurmans.
\newblock Dual representations for dynamic programming and reinforcement
  learning.
\newblock In \emph{2007 IEEE International Symposium on Approximate Dynamic
  Programming and Reinforcement Learning}, pages 44--51. IEEE, 2007.

\bibitem[Wang et~al.(2019{\natexlab{b}})Wang, Wang, Wu, and
  Zhang]{wang2019influence}
T.~Wang, J.~Wang, Y.~Wu, and C.~Zhang.
\newblock Influence-based multi-agent exploration.
\newblock \emph{arXiv preprint arXiv:1910.05512}, 2019{\natexlab{b}}.

\bibitem[Wang et~al.(2021)Wang, Shih, Xie, and Sadigh]{wang2021influencing}
W.~Z. Wang, A.~Shih, A.~Xie, and D.~Sadigh.
\newblock Influencing towards stable multi-agent interactions.
\newblock \emph{arXiv preprint arXiv:2110.08229}, 2021.

\bibitem[Wang and Bellemare(2019)]{wang2019lagged}
Y.~Wang and M.~F. Bellemare.
\newblock Lagged variables as instruments, 2019.

\bibitem[Wilkins(2018)]{wilkins2018lag}
A.~S. Wilkins.
\newblock To lag or not to lag?: Re-evaluating the use of lagged dependent
  variables in regression analysis.
\newblock \emph{Political Science Research and Methods}, 6\penalty0
  (2):\penalty0 393--411, 2018.

\bibitem[Xie et~al.(2020{\natexlab{a}})Xie, Harrison, and Finn]{xie2020deep}
A.~Xie, J.~Harrison, and C.~Finn.
\newblock Deep reinforcement learning amidst lifelong non-stationarity.
\newblock \emph{arXiv preprint arXiv:2006.10701}, 2020{\natexlab{a}}.

\bibitem[Xie et~al.(2020{\natexlab{b}})Xie, Losey, Tolsma, Finn, and
  Sadigh]{xie2020learning}
A.~Xie, D.~P. Losey, R.~Tolsma, C.~Finn, and D.~Sadigh.
\newblock Learning latent representations to influence multi-agent interaction.
\newblock \emph{arXiv preprint arXiv:2011.06619}, 2020{\natexlab{b}}.

\bibitem[Xie(2019)]{simglucose}
J.~Xie.
\newblock \emph{Simglucose v0.2.1 (2018)}, 2019.
\newblock URL \url{https://github.com/jxx123/simglucose}.

\bibitem[Xie et~al.(2019)Xie, Ma, and Wang]{xie2019towards}
T.~Xie, Y.~Ma, and Y.-X. Wang.
\newblock Towards optimal off-policy evaluation for reinforcement learning with
  marginalized importance sampling.
\newblock \emph{arXiv preprint arXiv:1906.03393}, 2019.

\bibitem[Xu et~al.(2020)Xu, Chen, Srinivasan, de~Freitas, Doucet, and
  Gretton]{xu2020learning}
L.~Xu, Y.~Chen, S.~Srinivasan, N.~de~Freitas, A.~Doucet, and A.~Gretton.
\newblock Learning deep features in instrumental variable regression.
\newblock \emph{arXiv preprint arXiv:2010.07154}, 2020.

\bibitem[Yang et~al.(2020)Yang, Nachum, Dai, Li, and Schuurmans]{yang2020off}
M.~Yang, O.~Nachum, B.~Dai, L.~Li, and D.~Schuurmans.
\newblock Off-policy evaluation via the regularized lagrangian.
\newblock \emph{Advances in Neural Information Processing Systems},
  33:\penalty0 6551--6561, 2020.

\bibitem[Yuan et~al.(2021)Yuan, Chandak, Giguere, Thomas, and
  Niekum]{yuan2021sope}
C.~Yuan, Y.~Chandak, S.~Giguere, P.~S. Thomas, and S.~Niekum.
\newblock Sope: Spectrum of off-policy estimators.
\newblock \emph{Advances in Neural Information Processing Systems},
  34:\penalty0 18958--18969, 2021.

\bibitem[Zhang and Lesser(2010)]{zhang2010multi}
C.~Zhang and V.~Lesser.
\newblock Multi-agent learning with policy prediction.
\newblock In \emph{Twenty-fourth AAAI conference on artificial intelligence},
  2010.

\bibitem[Zhou et~al.(2020)Zhou, Chen, Varshney, and
  Jagmohan]{zhou2020nonstationary}
H.~Zhou, J.~Chen, L.~R. Varshney, and A.~Jagmohan.
\newblock Nonstationary reinforcement learning with linear function
  approximation.
\newblock \emph{arXiv preprint arXiv:2010.04244}, 2020.

\end{thebibliography}
